\documentclass{article}

\usepackage[preprint]{neurips_2026}

\usepackage[utf8]{inputenc} % allow utf-8 input
\usepackage[T1]{fontenc}    % use 8-bit T1 fonts
\usepackage{hyperref}       % hyperlinks
\usepackage{url}            % simple URL typesetting
\usepackage{booktabs}       % professional-quality tables
\usepackage{amsfonts}       % blackboard math symbols
\usepackage{nicefrac}       % compact symbols for 1/2, etc.
\usepackage{microtype}      % microtypography
\usepackage{xcolor}         % colors

\title{
Fast Gauss-Newton for Multiclass Cross-Entropy
}

\author{
Mikalai Korbit  \qquad Mario Zanon \\
IMT School for Advanced Studies Lucca \\
Lucca, Italy \\
\texttt{\{mikalai.korbit,mario.zanon\}@imtlucca.it}
}

\usepackage{amsmath,amssymb,amsfonts,bm,mathtools}
\newcommand{\R}{\mathbb{R}}

\newcommand{\E}{\mathbb{E}}

\newcommand{\eps}{\epsilon}

\newcommand{\train}{\mathcal{D}}

\newcommand{\calB}{\mathcal{B}}

\newcommand{\vzero}{\bm{0}}

\newcommand{\vw}{\bm{w}}
\newcommand{\vx}{\bm{x}}
\newcommand{\vy}{\bm{y}}
\newcommand{\vz}{\bm{z}}
\newcommand{\vp}{\bm{p}}
\newcommand{\vr}{\bm{r}}
\newcommand{\vg}{\bm{g}}

\newcommand{\vd}{\bm{d}}
\newcommand{\vv}{\bm{v}}
\newcommand{\vu}{\bm{u}}
\newcommand{\vb}{\bm{b}}

\newcommand{\ve}{\bm{e}}
\newcommand{\va}{\bm{a}}
\newcommand{\vh}{\bm{h}}

\newcommand{\vbeta}{\bm{\beta}}
\newcommand{\vgamma}{\bm{\gamma}}
\newcommand{\vdelta}{\bm{\delta}}

\newcommand{\vrho}{\bm{\rho}}

\newcommand{\mI}{\bm{I}}
\newcommand{\mH}{\bm{H}}
\newcommand{\mJ}{\bm{J}}
\newcommand{\mQ}{\bm{Q}}
\newcommand{\mW}{\bm{W}}
\newcommand{\mX}{\bm{X}}

\newcommand{\mK}{\bm{K}}
\newcommand{\mB}{\bm{B}}
\newcommand{\mP}{\bm{P}}
\newcommand{\mA}{\bm{A}}

\newcommand{\mR}{\bm{R}}
\newcommand{\mZ}{\bm{Z}}

\DeclareMathOperator{\Tr}{Tr}

\DeclareMathOperator{\diag}{diag}

\newcommand{\softmax}{\mathrm{softmax}}

\newcommand{\norm}[1]{\left\lVert #1 \right\rVert}

\newcommand{\Ls}{\mathcal{L}}  % generic loss

\newcommand{\lr}{\alpha}       % learning rate (if used consistently)

\newcommand{\Risk}{\mathcal{R}} % population / empirical risk
\newcommand{\cF}{\mathcal{F}}      % filtration

\usepackage{amsthm}

\theoremstyle{plain}
\newtheorem{theorem}{Theorem}[section]
\newtheorem{proposition}[theorem]{Proposition}
\newtheorem{lemma}[theorem]{Lemma}

\theoremstyle{definition}

\newtheorem{assumption}[theorem]{Assumption}

\theoremstyle{remark}

\usepackage{algorithm}
\usepackage{algpseudocode}

\algrenewcommand\algorithmicrequire{\textbf{Input:}}
\algrenewcommand\algorithmicensure{\textbf{Output:}}

\usepackage{tikz}
\usetikzlibrary{arrows.meta, positioning, calc, shapes.geometric, fit, backgrounds, shadows, shapes}

\begin{document}

\maketitle

\begin{abstract}

In multiclass softmax cross-entropy, the full generalized Gauss-Newton (GGN) curvature couples all output logits through the softmax covariance, making curvature-vector products harder to scale as the number of classes grows. 
We show that the standard multiclass GGN can be decomposed exactly into a true-vs-rest term and a positive semidefinite within-competitor covariance term. 
Fast Gauss-Newton (FGN) retains the first term and drops the second, yielding a positive semidefinite under-approximation of the multiclass GGN that is exact for binary classification. 
The derivation uses an exact true-vs-rest scalar-margin representation of softmax cross-entropy: the loss and gradient are unchanged, and the approximation enters only at the curvature level. 
Exploiting the FGN curvature structure, the damped update can be written as an equivalent whitened row-space system with one row per mini-batch example. 
We solve this system matrix-free by conjugate gradient using Jacobian-vector and vector-Jacobian products of the scalar margin map. 
Targeted mechanism experiments show how the dropped covariance grows with competitor dispersion and how damping can reduce its effect on the resulting steps.
In frozen-feature head optimization and end-to-end CIFAR-100 training, FGN achieves similar test accuracy to a full GGN baseline with faster updates.

\end{abstract}

%%%%%%%%%%%%%%%%%%%%%%%%%%%%%%%%%%%%%%%%%%%%%%%%%%%%%%%%%%%%%%%%%%%%%%%%%%%%%%%
%%%%%%%%%%%%%%%%%%%%%%%%%%%%%%%%%%%%%%%%%%%%%%%%%%%%%%%%%%%%%%%%%%%%%%%%%%%%%%%

\section{Introduction}
\label{sec:intro}

The behavior and computational tractability of second-order methods 
are fundamentally governed by the structure of their curvature models. 
For neural network losses, the generalized Gauss-Newton (GGN) matrix provides 
a canonical positive semidefinite curvature approximation. 
In multiclass softmax cross-entropy, however,
it inherits the full $C$-dimensional output geometry.
For logits $\vz(\vw)\in\R^C$, 
let $\mJ_z\in\R^{C\times d}$ be the logit Jacobian
with respect to the parameters and 
let $\vp\in\R^C$ be the softmax probability vector.
The corresponding GGN matrix is
\[
\mJ_z^\top\bigl(\diag(\vp)-\vp\vp^\top\bigr)\mJ_z,
\]
so even matrix-free solvers 
still have to apply the $C$-dimensional softmax covariance 
in their curvature-vector products 
unless the output-space curvature model itself is changed~\citep{schraudolph2002fast,martens2010deep,bottou2018optimization}.
We address this bottleneck by changing the curvature construction
while leaving the softmax loss and gradient unchanged.

The central identity behind FGN is an exact decomposition of the multiclass
softmax GGN:
\[
\mH^{\mathrm{GGN}}
=
\mH^{\mathrm{FGN}}
+
\mR^{\mathrm{comp}},
\qquad
\mR^{\mathrm{comp}} \succeq 0.
\]
Here $\mH^{\mathrm{FGN}}$ is 
the retained true-vs-rest curvature term, 
and $\mR^{\mathrm{comp}}$ is a positive semidefinite (PSD) covariance term 
over the competing logits. 
FGN keeps the first term and drops the second, 
yielding a PSD under-approximation of the multiclass GGN 
that is exact in the binary case.

The retained term comes from an exact scalar true-vs-rest representation of the softmax cross-entropy, 
defined formally in Section~\ref{sec:margin}.
This yields one Jacobian row of the scalar margin per example, 
making the damped FGN update
reducible to a $b$-dimensional row-space solve.
The resulting matrix-free CG solver 
retains the structure of Gauss-Newton methods 
while avoiding the class-coupled curvature products 
required by the full softmax GGN.

FGN is a structured, regime-dependent curvature approximation. 
The approximation is tighter when 
one clear hardest negative dominates 
the competitor distribution, 
so that the dropped covariance term is small.
When probability mass is spread across several competitors, 
$\mR^{\mathrm{comp}}$ can be larger.
This affects approximation tightness, not the validity of FGN: 
the decomposition holds for every finite-logit example.

This positioning distinguishes FGN from nearby second-order methods.
In solver structure, 
FGN is closest to Hessian-free optimization~\citep{martens2011learning,wiesler2013investigations,kiros2013training},
stochastic Gauss-Newton methods such as SGN~\citep{gargiani2020promise},
and row-space formulations based on matrix inversion 
identities~\citep{ren2019efficient,korbit2025exact}.
It shares the matrix-free linear algebra, 
but not the curvature construction: 
those methods still operate on the full multiclass output geometry. 
Compared with diagonal or structured preconditioners such as 
AdaHessian~\citep{yao2021adahessian}, 
Sophia~\citep{liu2023sophia}, 
K-FAC~\citep{martens2015optimizing}, 
and Shampoo~\citep{gupta2018shampoo},
FGN is not another factorization of the same softmax GGN/Fisher. 
Instead, it exploits a loss-specific decomposition 
that changes the output-space curvature before the row-space solve. 
Unlike sampled-output approximations or multiclass-to-binary training 
reductions~\citep{chen2016strategies,allwein2000reducing,rifkin2004defense}, 
FGN leaves the original softmax loss and exact gradient unchanged; 
the approximation begins only at the GGN curvature level.
Appendix~\ref{apx:related_work} discusses these connections in more detail.

Our contributions are as follows.
\begin{itemize}
    \item We derive an exact decomposition of the standard softmax GGN into a true-vs-rest term and a PSD 
    covariance term over competing logits. 
    We define FGN by retaining the true-vs-rest term 
    while preserving the exact softmax loss and gradient. 
    The resulting curvature is a PSD under-approximation of the multiclass GGN, 
    is exact in the binary case, 
    % and has an explicit parameter-space gap given by 
    % a weighted covariance of competitor Jacobian rows.
    and differs from the full GGN by an explicit weighted covariance 
    of competitor Jacobian rows.

    \item We derive 
    % an exactly equivalent 
    a whitened $b$-dimensional row-space system for the damped FGN step and 
    a matrix-free conjugate gradient solver 
    based on scalar-margin Jacobian-vector and vector-Jacobian products.

    \item We test what the approximation discards and what it gains: 
    how the dropped covariance affects damped steps, 
    and whether the resulting solver can reduce update cost while preserving test accuracy, from frozen-feature head optimization to end-to-end ResNet-20 training on CIFAR-100.

\end{itemize}

%%%%%%%%%%%%%%%%%%%%%%%%%%%%%%%%%%%%%%%%%%%%%%%%%%%%%%%%%%%%%%%%%%%%%%%%%%%%%%%
%%%%%%%%%%%%%%%%%%%%%%%%%%%%%%%%%%%%%%%%%%%%%%%%%%%%%%%%%%%%%%%%%%%%%%%%%%%%%%%

\section{True-vs-Rest Margin Setup}
\label{sec:margin}

We use a standard reformulation of multiclass softmax cross-entropy 
as a softplus of a true-vs-rest margin; 
see, e.g., \citet{kobayashi2019large}.
This section introduces the notation and identities used in 
the curvature decomposition in Section~\ref{sec:fgn}.

Consider one training example $(\vx,\vy)$, 
parameters $\vw\in\R^d$, and
$C\ge 2$ classes. Let $\vy\in\{0,1\}^C$ be one-hot with true class
denoted by $\star$. 
For class-indexed quantities, the subscript $-\star$ denotes the
competitor block obtained by removing the true-class entry. 
The symbol $\dagger$ denotes an aggregate competitor quantity, 
not an additional class:
$z_\dagger$ is the log-sum-exp aggregate of the competitor logits and
$p_\dagger$ is their total probability mass.

The model outputs logits $\vz(\vw)=f(\vx;\vw)\in\R^C$, 
with class components $z_c(\vw)$ for $c=1,\ldots,C$; 
the induced probabilities are $\vp=\softmax(\vz)$. 
Gradients of scalar-valued functions are column vectors;
their transposes are Jacobian rows. 
Unless a derivative with respect to
logits is written explicitly as $\nabla_{\vz}$, 
Jacobians are with respect to the parameters $\vw$. 
% The logit Jacobian is
% \begin{align}
% \mJ_z(\vw)
% \coloneqq
% \frac{\partial \vz(\vw)}{\partial \vw}
% =
% \begin{bmatrix}
% \nabla_{\vw} z_1(\vw)^\top \\
% \vdots \\
% \nabla_{\vw} z_C(\vw)^\top
% \end{bmatrix}
% \in \R^{C\times d}.
% \label{eq:fgn_logit_jacobian}
% \end{align}
The logit Jacobian is
$\mJ_z(\vw)\coloneqq \partial \vz(\vw)/\partial \vw\in\R^{C\times d}$;
its $c$-th row is $\nabla_{\vw}z_c(\vw)^\top$.
Throughout, logits are assumed finite, 
and derivative identities are understood at points 
where the required derivatives exist.

% The softmax cross-entropy loss is
% \begin{align}
% \ell(\vw;\vx,\vy)
% =
% -\sum_{c=1}^C y_c \log p_c(\vw)
% =
% -\log p_\star(\vw).
% \label{eq:fgn_ce_loss}
% \end{align}
Define the aggregate competitor logit and 
aggregate competitor probability by
\begin{align}
z_\dagger(\vw)
&\coloneqq
\log \sum_{j\neq \star} e^{z_j(\vw)},
\label{eq:fgn_zdagger_def}
\\
p_\dagger(\vw)
&\coloneqq
\sum_{j\neq \star} p_j(\vw)
=
\frac{e^{z_\dagger(\vw)}}{e^{z_\star(\vw)}+e^{z_\dagger(\vw)}}
=
1-p_\star(\vw),
\label{eq:fgn_p_star_dagger}
\end{align}
and the true-vs-rest margin by
\begin{align}
s(\vw)
\coloneqq
z_\dagger(\vw)-z_\star(\vw).
\label{eq:fgn_margin_def}
\end{align}

With $\phi(s)\coloneqq \log(1+e^s)$, 
the multiclass cross-entropy loss 
can be written exactly as a softplus of the true-vs-rest
margin:
\begin{align}
\ell(\vw;\vx,\vy)
&=
-\sum_{c=1}^C y_c\log p_c(\vw)
=
-\log p_\star(\vw)
=
\log\bigl(1+e^{s(\vw)}\bigr)
=
\phi(s(\vw)).
\label{eq:fgn_ce_loss}
% \label{eq:fgn_scalar_loss}
\end{align}
This is an exact representation of the loss: all competing logits
still enter through the log-sum-exp aggregate $z_\dagger$.

To prepare the Gauss-Newton curvature and the gradient identity, 
we now differentiate the outer softplus link and the inner margin map.
The derivatives of the softplus link are
\begin{align}
\phi'(s)
&=
\frac{e^s}{1+e^s}
=
p_\dagger,
&
\phi''(s)
&=
\frac{e^s}{(1+e^s)^2}
=
p_\star p_\dagger.
\label{eq:fgn_phi_derivs}
\end{align}
For finite logits, $p_\dagger>0$. Define the conditional competitor
distribution $\vrho\in\R^{C-1}$ by
\begin{align}
\rho_j
\coloneqq
\frac{p_j}{p_\dagger},
\qquad
j\neq \star.
\label{eq:fgn_rho_def}
\end{align}
Thus $\vrho$ is the softmax distribution restricted to the competitor classes
and renormalized.

Order the logit coordinates as $(z_\star,\vz_{-\star})$. Differentiating
$z_\dagger$ and $s=z_\dagger-z_\star$ with respect to the logits gives
\begin{align}
\nabla_{\vz} z_\dagger
&=
\begin{bmatrix}
0 \\
\vrho
\end{bmatrix},
\qquad
\nabla_{\vz} s
=
\begin{bmatrix}
-1 \\
\vrho
\end{bmatrix}.
\label{eq:fgn_margin_logit_grad}
\end{align}
Equivalently,
\begin{align}
p_\dagger \nabla_{\vz} s
=
\vp-\vy,
\label{eq:fgn_logit_grad_identity}
\end{align}
where $\vp-\vy$ is written in the same reordered logit coordinates.

Define the margin Jacobian row
\begin{align}
\mJ_s(\vw)
\coloneqq
\nabla_{\vw}s(\vw)^\top
=
(\nabla_{\vz} s)^\top \mJ_z
\in \R^{1\times d}.
\label{eq:fgn_margin_jac}
\end{align}
Using \eqref{eq:fgn_margin_logit_grad}, this row has the explicit form
\begin{align}
\mJ_s
&=
-\nabla_{\vw} z_\star^\top
+
\sum_{j\neq \star} \rho_j\,\nabla_{\vw} z_j^\top.
\label{eq:fgn_margin_jac_expanded}
\end{align}
Thus $\mJ_s$ is the $\vrho$-weighted mean competitor Jacobian row minus the
true-class Jacobian row.

Combining the scalar chain rule with
\eqref{eq:fgn_logit_grad_identity} gives
\begin{align}
\nabla_{\vw}\ell(\vw;\vx,\vy)
=
\phi'(s)\nabla_{\vw}s
=
p_\dagger\,\mJ_s^\top
=
\mJ_z^\top(\vp-\vy).
\label{eq:fgn_grad_softmax_equiv}
\end{align}
The final expression is exactly the standard softmax cross-entropy gradient.
Hence the scalar-margin representation changes 
neither the softmax loss nor its gradient. 
Although the loss factors through the margin $s$, 
its geometry remains multiclass through $z_\dagger$ and $\vrho$.
Section~\ref{sec:fgn} uses this structure to
define and decompose the FGN curvature; expanded algebra is provided in
Appendix~\ref{apx:margin_derivations}.

%%%%%%%%%%%%%%%%%%%%%%%%%%%%%%%%%%%%%%%%%%%%%%%%%%%%%%%%%%%%%%%%%%%%%%%%%%%%%%%
%%%%%%%%%%%%%%%%%%%%%%%%%%%%%%%%%%%%%%%%%%%%%%%%%%%%%%%%%%%%%%%%%%%%%%%%%%%%%%%

\section{FGN as a Structured Approximation to Multiclass GGN}
\label{sec:fgn}

\subsection{Full softmax GGN and the retained margin term}

The true-vs-rest reformulation preserves the softmax loss and gradient exactly. This does not mean it preserves the Gauss-Newton curvature: 
the latter depends on the choice of inner map. 
The standard multiclass GGN uses the full logit map; 
FGN uses the scalar margin. 
We now make their difference explicit.

For a single example, the classical multiclass GGN for softmax cross-entropy is
\begin{align}
\mH^{\mathrm{GGN}}(\vw)
&\coloneqq
\mJ_z^\top
\bigl(\diag(\vp)-\vp\vp^\top\bigr)
\mJ_z,
\label{eq:h_ggn}
\end{align}
where $\mJ_z\in\R^{C\times d}$ is the logit Jacobian 
from Section~\ref{sec:margin}.
The factor $\diag(\vp)-\vp\vp^\top$ is the standard output-space curvature 
for multinomial log-likelihoods
and coincides with the Fisher factor in logit space
\citep{schraudolph2002fast,bottou2018optimization}.

Applying the Gauss-Newton construction instead to the exact scalar composition
$\ell(\vw)=\phi(s(\vw))$ from Section~\ref{sec:margin} yields
\begin{align}
\mH^{\mathrm{FGN}}(\vw)
&\coloneqq
\phi''\bigl(s(\vw)\bigr)\,\mJ_s(\vw)^\top \mJ_s(\vw)
=
p_\star p_\dagger\,\mJ_s^\top \mJ_s.
\label{eq:fgn_h_single}
\end{align}
Each example therefore contributes a PSD matrix of rank at most one,
formed from the true-vs-rest margin Jacobian.

Two chain-rule decompositions of the exact Hessian make the distinction explicit:
\begin{align*}
\nabla_{\vw}^2 \ell
&=
\underbrace{
\mJ_z^\top
\bigl(\diag(\vp)-\vp\vp^\top\bigr)
\mJ_z
}_{\mH^{\mathrm{GGN}}(\vw)}
+
\sum_{c=1}^C (p_c-y_c)\,\nabla_{\vw}^2 z_c,
\\
\nabla_{\vw}^2 \ell
&=
\underbrace{
p_\star p_\dagger\,\mJ_s^\top \mJ_s
}_{\mH^{\mathrm{FGN}}(\vw)}
+
p_\dagger\,\nabla_{\vw}^2 s.
\end{align*}
Thus, relative to the exact Hessian, 
both matrices are Gauss-Newton curvatures 
obtained by dropping second-derivative terms, 
but they correspond to different inner maps: 
the full logit map $\vz(\vw)$ and the scalar margin map $s(\vw)$. 
Their difference is therefore already present in output-space curvature, 
and the next subsection identifies it precisely.

\subsection{Decomposition of the softmax GGN}

Using the coordinate order $(z_\star,\vz_{-\star})$ from
Section~\ref{sec:margin}, 
write
\begin{align*}
\vz
&=
\begin{bmatrix}
z_\star \\
\vz_{-\star}
\end{bmatrix},
&
\mJ_z
&=
\begin{bmatrix}
\mJ_\star \\
\mJ_{-\star}
\end{bmatrix},
&
\vp
&=
\begin{bmatrix}
p_\star \\
\vp_{-\star}
\end{bmatrix}
=
\begin{bmatrix}
p_\star \\
p_\dagger\vrho
\end{bmatrix}.
\end{align*}
Here $\mJ_\star\in\R^{1\times d}$ and
$\mJ_{-\star}\in\R^{(C-1)\times d}$. 
By \eqref{eq:fgn_margin_logit_grad} and \eqref{eq:fgn_margin_jac},
\begin{align*}
\nabla_{\vz} s
&=
\begin{bmatrix}
-1 \\
\vrho
\end{bmatrix},
&
\mJ_s
&=
(\nabla_{\vz} s)^\top \mJ_z
=
-\mJ_\star + \vrho^\top \mJ_{-\star}.
\end{align*}

\begin{theorem}[Decomposition of the multiclass softmax GGN]
\label{thm:fgn_ggn_decomposition}
For a single example with finite logits and $C\ge 2$, 
in the coordinate order $(z_\star,\vz_{-\star})$, 
the softmax covariance admits the decomposition
\begin{align}
\diag(\vp)-\vp\vp^\top
&=
p_\star p_\dagger\,\nabla_{\vz}s\,\nabla_{\vz}s^\top
+
p_\dagger
\begin{bmatrix}
0 & 0 \\
0 & \diag(\vrho)-\vrho\vrho^\top
\end{bmatrix}.
\label{eq:fgn_softmax_cov_decomposition}
\end{align}
Consequently,
\begin{align}
\mH^{\mathrm{GGN}}
&=
\mH^{\mathrm{FGN}}
+
p_\dagger\,\mJ_{-\star}^\top
\bigl(\diag(\vrho)-\vrho\vrho^\top\bigr)
\mJ_{-\star}.
\label{eq:fgn_ggn_decomposition}
\end{align}
\end{theorem}

\noindent
\textit{Proof.}
Appendix~\ref{apx:fgn_decomposition}.

\begin{figure*}[t]
\centering
\begin{tikzpicture}[
    >=Stealth,
    font=\small,
    box/.style={draw=black!75, line width=0.9pt},
    split/.style={draw=black!45, line width=0.6pt},
    kept/.style={draw=blue!60!black, line width=1pt},
    dropped/.style={draw=gray!50, dashed, line width=0.9pt},
    blocklabel/.style={align=center, inner sep=1pt},
    titlelabel/.style={font=\footnotesize\bfseries, align=center, anchor=south},
    axlabel/.style={font=\scriptsize},
    eqlabel/.style={font=\Large}
]

% Matrix geometry
\def\W{2.8}
\def\H{2.8}
\def\a{0.75}   % true-class block size
\def\b{2.05}   % competitor block size

% x-shifts
\def\xA{0.0}
\def\xB{4.8}
\def\xC{9.6}

% =========================================================
% 1. Full softmax covariance factor
% =========================================================
\begin{scope}[shift={(\xA,0)}]
    \draw[box, fill=gray!8] (0,0) rectangle (\W,\H);
    \draw[split] (\a,0) -- (\a,\H);
    \draw[split] (0,\b) -- (\W,\b);

    % axis labels
    \node[axlabel, left]  at (0,\b+0.5*\a) {$\star$};
    \node[axlabel, left]  at (0,0.5*\b) {$-\star$};
    \node[axlabel, above] at (0.5*\a,\H) {$\star$};
    \node[axlabel, above] at (\a+0.5*\b,\H) {$-\star$};

    % block labels 
    \node[blocklabel, font=\scriptsize] at (0.5*\a,\b+0.5*\a) {$p_\star p_\dagger$};
    \node[blocklabel, font=\scriptsize] at (\a+0.5*\b,0.5*\b) {competitor \\ block};
    
    % title
    \node[titlelabel, text=black!70] at (0.5*\W, \H+0.45) {Softmax \\ covariance factor};
\end{scope}

\node[eqlabel] at (3.8,0.5*\H) {$=$};

% =========================================================
% 2. Kept rank-one true-vs-rest term
% =========================================================
\begin{scope}[shift={(\xB,0)}]
    % Base box (lightened to prevent "dense matrix" misread)
    \draw[kept, fill=blue!4] (0,0) rectangle (\W,\H);
    \draw[split, blue!40, dashed] (\a,0) -- (\a,\H);
    \draw[split, blue!40, dashed] (0,\b) -- (\W,\b);

    % outer-product cue
    \draw[fill=blue!15, draw=blue!60!black, rounded corners=1pt]
        (-0.32,0) rectangle (-0.12,\H);
    \draw[fill=blue!15, draw=blue!60!black, rounded corners=1pt]
        (0,\H+0.12) rectangle (\W,\H+0.32);
    \node[font=\scriptsize, text=blue!75!black] at (-0.22,\H+0.22) {$\times$};

    \node[blocklabel, font=\scriptsize, text=blue!80!black] at (0.5*\W,0.5*\H)
        {\textbf{rank-one}\\[-0.1ex]\textbf{outer product} \\ [1.0ex] $p_\star p_\dagger\,\nabla_{\vz}s\,\nabla_{\vz}s^\top$};

    % title
    \node[titlelabel, text=blue!70!black] at (0.5*\W, \H+0.45) {Kept\\true-vs-rest term};
\end{scope}

\node[eqlabel] at (8.6,0.5*\H) {$+$};

% =========================================================
% 3. Dropped within-competitor covariance
% =========================================================
\begin{scope}[shift={(\xC,0)}]
    % Dashed outer boundary to show it's discarded
    \draw[dropped] (0,0) rectangle (\W,\H);
    \draw[split, dashed, gray!50] (\a,0) -- (\a,\H);
    \draw[split, dashed, gray!50] (0,\b) -- (\W,\b);

    % explicit zero blocks
    \node[font=\small, text=black!45] at (0.5*\a,\b+0.5*\a) {$\mathbf{0}$};
    \node[font=\small, text=black!45] at (\a+0.5*\b,\b+0.5*\a) {$\mathbf{0}$};
    \node[font=\small, text=black!45] at (0.5*\a,0.5*\b) {$\mathbf{0}$};

    % active competitor block (dashed and light to match dropped semantics)
    \draw[fill=gray!8, draw=gray!60, dashed, line width=0.8pt]
        (\a,0) rectangle (\W,\b);

    \node[blocklabel, font=\scriptsize] at (\a+0.5*\b,0.5*\b)
        {$p_\dagger\bigl(\diag(\vrho)$\\[-0.1ex]$-\,\vrho\vrho^\top\bigr)$};

    % title
    \node[titlelabel, text=black!70] at (0.5*\W, \H+0.45) {Dropped competitor\\covariance};
\end{scope}

\end{tikzpicture}
\caption{\textbf{Output-space decomposition underlying FGN.}
The softmax covariance factor splits into a kept rank-one true-vs-rest term
and a dropped within-competitor covariance term.
}
\label{fig:fgn_covariance_decomposition}
\end{figure*}

\paragraph{Interpretation.}
Theorem~\ref{thm:fgn_ggn_decomposition} identifies the approximation precisely.
The kept term is the exact true-vs-rest covariance induced by the scalar
margin $s=z_\dagger-z_\star$, while the residual is a covariance supported
entirely on the competitor block. 
Figure~\ref{fig:fgn_covariance_decomposition} visualizes the output-space decomposition in Theorem~\ref{thm:fgn_ggn_decomposition}; 
multiplication by $\mJ_z^\top$ and $\mJ_z$ yields 
the corresponding parameter-space decomposition.
Thus FGN is not merely a cheaper implementation of the full softmax GGN:
it is a different PSD curvature obtained by removing the within-competitor
covariance term.

% Importantly, FGN does not collapse the competitor set to only the top
% competing class.
FGN, like full softmax GGN, retains all competitors.
Since
$\mJ_s = -\mJ_\star + \vrho^\top \mJ_{-\star}$,
FGN uses all competitor logits through $\vrho$ and retains the exact
$\vrho$-weighted mean competitor direction.
% What it discards is the covariance of individual competitor directions around
% that mean.
% This distinction matters in large-class problems: 
% FGN summarizes the competitor contribution by its conditional mean direction,
% rather than by the full within-competitor covariance.
This differs from a hard top-1 approximation, 
which would replace that mean by a single selected competitor.
Near a tie, such a selection can change abruptly as the logits vary,
whereas $\vrho$ and the corresponding weighted mean vary smoothly.
Full GGN retains both the mean contribution and the
within-competitor covariance; 
FGN drops only the latter.

\subsection{
Consequences: PSD under-approximation, exactness, and gap characterization
}

The decomposition has several immediate consequences.

\begin{proposition}[Consequences of the decomposition]
\label{prop:fgn_decomposition_consequences}
For a single example with finite logits and $C\ge 2$, the following hold.
\begin{enumerate}
    \item $0 \preceq \mH^{\mathrm{FGN}} \preceq \mH^{\mathrm{GGN}}$.

    \item If $C=2$, then $\mH^{\mathrm{FGN}}=\mH^{\mathrm{GGN}}$.

    \item The dropped output-space residual in \eqref{eq:fgn_softmax_cov_decomposition} vanishes
if and only if $\vrho$ is one-hot.

    \item If $\mJ_1,\dots,\mJ_{C-1}$ are the rows of $\mJ_{-\star}$ and
    $\bar{\mJ}_{\rho}\coloneqq \vrho^\top \mJ_{-\star}$, then
    \begin{align}
    \mH^{\mathrm{GGN}}-\mH^{\mathrm{FGN}}
    &=
    p_\dagger
    \sum_{j=1}^{C-1}
    \rho_j\,
    (\mJ_j-\bar{\mJ}_{\rho})^\top
    (\mJ_j-\bar{\mJ}_{\rho}).
    \label{eq:fgn_parameter_space_gap}
    \end{align}
\end{enumerate}
\end{proposition}

\noindent
\textit{Proof.}
Appendix~\ref{apx:fgn_consequences}.

Item~1 shows that FGN is always a PSD under-approximation of the multiclass GGN.
Item~2 is the binary special case.
The third item characterizes exactness in output space. 
The dropped covariance vanishes exactly 
when the conditional competitor distribution $\vrho$ is one-hot; 
in that case the residual in~\eqref{eq:fgn_softmax_cov_decomposition} 
is already zero before projection,
so $\mH^{\mathrm{FGN}}=\mH^{\mathrm{GGN}}$. 
For finite logits and $C>2$, 
this condition is never attained because all
competitor probabilities are positive, 
but it is approached as one competitor dominates. 
After projection to parameter space, 
equality can hold even when the competitor distribution is diffuse: 
the competitor Jacobian can annihilate a nonzero output-space residual, 
e.g., when all competitor Jacobian rows coincide.
Finally, Item 4 identifies the dropped parameter-space term 
as a $\vrho$-weighted covariance of competitor Jacobian rows around their mean.
The curvature gap therefore depends on both competitor probabilities 
and how differently their logits respond to parameter changes.

A useful per-example dispersion diagnostic is
\begin{align}
\xi
&\coloneqq
1-\norm{\vrho}_2^2
=
\Tr\bigl(\diag(\vrho)-\vrho\vrho^\top\bigr).
\label{eq:fgn_xi}
\end{align}
It satisfies $0 \le \xi \le 1-\frac{1}{C-1}$,
with $\xi=0$ if and only if $\vrho$ is one-hot, and the upper bound
attained by the uniform competitor distribution.

For interpretation, define the effective number of competitors as
\[
N_{\mathrm{eff}}
\coloneqq
\frac{1}{\norm{\vrho}_2^2}
=
\frac{1}{1-\xi}.
\]
If the conditional competitor mass is approximately uniform over a fixed
subset of $k$ classes and negligible elsewhere, 
then
$N_{\mathrm{eff}}\approx k$ and $\xi\approx 1-1/k$.
Thus the total number of classes alone does not determine
competitor dispersion.

The quantity $p_\dagger \xi$ is the exact trace 
of the dropped logit-space residual in~\eqref{eq:fgn_softmax_cov_decomposition}. 
It separates two effects: 
$p_\dagger$ measures the total probability mass on the wrong classes, 
while $\xi$ 
indicates whether that mass is concentrated 
on one competitor or spread across many competitors.
The dropped trace can therefore be small 
either because the model assigns
high probability to the true class and $p_\dagger$ is small,
or because one competitor dominates and $\xi$ is small.
The parameter-space curvature gap 
$\mH^{\mathrm{GGN}}-\mH^{\mathrm{FGN}}$
further depends on the geometry 
of the competitor Jacobian rows through~\eqref{eq:fgn_parameter_space_gap}.

\subsection{Batched form and row-space structure}
\label{sec:batched_form}

The single-example identities above assemble cleanly at the mini-batch level.
For a mini-batch $\calB=\{(\vx_i,\vy_i)\}_{i=1}^b$, 
define
\begin{align}
\Ls(\vw;\calB)
&=
\frac{1}{b}\sum_{i=1}^b \ell(\vw;\vx_i,\vy_i),
\label{eq:fgn_batch_loss_def}
\end{align}
and write $q_i(\vw)\coloneqq p_\star^{(i)}(\vw)p_\dagger^{(i)}(\vw)$.
Collect the per-example quantities into
\begin{align*}
\mJ(\vw)
&\coloneqq
\begin{bmatrix}
\nabla_{\vw}s_1(\vw)^\top\\
\vdots\\
\nabla_{\vw}s_b(\vw)^\top
\end{bmatrix},
&
\vr(\vw)
&\coloneqq
\begin{bmatrix}
p_\dagger^{(1)}(\vw)\\
\vdots\\
p_\dagger^{(b)}(\vw)
\end{bmatrix},
&
\mQ(\vw)
&\coloneqq
\diag\!\bigl(
q_1(\vw),\dots,q_b(\vw)
\bigr).
\end{align*}
Then
\begin{align}
\nabla_{\vw}\Ls(\vw;\calB)
=
\frac{1}{b}\,\mJ(\vw)^\top \vr(\vw),
% \label{eq:fgn_batch_grad}
% \\[0.2em]
\qquad
\mH^{\mathrm{FGN}}(\vw;\calB)
=
\frac{1}{b}\,\mJ(\vw)^\top \mQ(\vw)\,\mJ(\vw).
% \label{eq:fgn_batch_h}
\label{eq:fgn_batch_grad_h}
\end{align}
For finite logits, $\mQ(\vw)$ is diagonal and strictly positive, with
entries in $(0,1/4]$.

This batched factorization is the structural payoff of the decomposition.
The mini-batch gradient remains exact, while the curvature acts entirely
through the row space of $\mJ(\vw)$ with a diagonal positive middle factor
$\mQ(\vw)$.
Unlike the softmax covariance $\diag(\vp)-\vp\vp^\top$, which is singular
when $C>1$, the batch matrix $\mQ(\vw)$ is diagonal and strictly positive
for finite logits.
This is the structure exploited by the whitened row-space solver
in Section~\ref{sec:solver}.

%%%%%%%%%%%%%%%%%%%%%%%%%%%%%%%%%%%%%%%%%%%%%%%%%%%%%%%%%%%%%%%%%%%%%%%%%%%%%%%
%%%%%%%%%%%%%%%%%%%%%%%%%%%%%%%%%%%%%%%%%%%%%%%%%%%%%%%%%%%%%%%%%%%%%%%%%%%%%%%

\section{Matrix-Free Row-Space Solver for FGN}
\label{sec:solver}

Starting from the batched factorization in~\eqref{eq:fgn_batch_grad_h},
the damped FGN step can be reduced to a $b$-dimensional row-space solve.
Whitening by $\mQ(\vw)^{1/2}$ gives a symmetric positive definite (SPD) system
that is directly compatible with matrix-free CG.
The row-space reduction introduces no additional curvature approximation: 
the only approximation relative to full softmax GGN is the
replacement of $\mH^{\mathrm{GGN}}$ by $\mH^{\mathrm{FGN}}$.

\subsection{From the damped parameter-space system to the whitened row-space system}
\label{sec:fgn_row_system}

Fix the current parameters $\vw$, mini-batch $\calB$, and damping $\lambda>0$, and let, for ease of notation,
\[
\mJ=\mJ(\vw),\qquad
\vr=\vr(\vw),\qquad
\mQ=\mQ(\vw),\qquad
\vg=\nabla_{\vw}\Ls(\vw;\calB),\qquad
\mH^{\mathrm{FGN}}=\mH^{\mathrm{FGN}}(\vw;\calB).
\]
Using \eqref{eq:fgn_batch_grad_h},
the damped FGN step is defined by
\begin{align}
\bigl(\mH^{\mathrm{FGN}}+\lambda \mI_{d} \bigr)\vd
&=
-\vg.
\label{eq:fgn_damped_system_param}
\end{align}
Moreover,
\[
\vg \in \mathrm{range}(\mJ^\top),
\qquad
\mathrm{range}(\mH^{\mathrm{FGN}})\subseteq \mathrm{range}(\mJ^\top).
\]
Therefore, for any $\lambda>0$, 
the unique solution of~\eqref{eq:fgn_damped_system_param} 
also lies in $\mathrm{range}(\mJ^\top)$.
We parameterize the direction as
\begin{align}
\vd
=
-\mJ^\top \vdelta,
\qquad
\vdelta \in \R^b.
\label{eq:fgn_row_direction_form}
\end{align}
Substituting \eqref{eq:fgn_row_direction_form} into
\eqref{eq:fgn_damped_system_param}, and using 
$\mH^{\mathrm{FGN}}=(1/b)\mJ^\top\mQ \mJ$, 
gives
\begin{align}
\bigl(\mH^{\mathrm{FGN}}+\lambda\mI_{d} \bigr)\vd
&=
-\frac{1}{b}\,\mJ^\top
\bigl(\mQ \mJ \mJ^\top + b\,\lambda \mI_b\bigr)\vdelta.
\end{align}
Since $\vg=(1/b)\mJ^\top\vr$, 
it is sufficient to solve the unwhitened
row-space system
\begin{align}
\bigl(\mQ \mJ \mJ^\top + b\,\lambda \mI_b\bigr)\vdelta
&=
\vr.
\label{eq:fgn_row_system_unwhitened}
\end{align}
This system is generally non-symmetric. 
Since $\mQ$ is diagonal with entries
$q_i=p_\star^{(i)}p_\dagger^{(i)}\in(0,1/4]$, 
$\mQ^{\pm 1/2}$ are well defined and cheap to apply.
Define the whitened quantities
\begin{align}
\tilde{\mJ}
&\coloneqq
\mQ^{1/2}\mJ,
&
\tilde{\vr}
&\coloneqq
\mQ^{-1/2}\vr,
&
\vu
&\coloneqq
\mQ^{-1/2}\vdelta,
&
\mK
&\coloneqq
\tilde{\mJ}\tilde{\mJ}^\top
\in\R^{b\times b}.
\end{align}
Then the equivalent whitened system is
\begin{align}
\bigl(\mK+b\,\lambda\mI_b\bigr)\vu
&=
\tilde{\vr}.
\label{eq:fgn_row_system_whitened}
\end{align}
Since $\mK \succeq 0$ and $\lambda>0$,
the operator in \eqref{eq:fgn_row_system_whitened}
is symmetric positive definite and directly compatible with CG.
The parameter-space direction is recovered by
\begin{align}
\vd
=
-\tilde{\mJ}^\top\vu
=
-\mJ^\top\mQ^{1/2}\vu.
\label{eq:fgn_direction_from_u}
\end{align}
Indeed, if \eqref{eq:fgn_row_system_whitened} holds, then
\[
\bigl(\mH^{\mathrm{FGN}}+\lambda\mI_{d} \bigr)(-\tilde{\mJ}^\top\vu)
=
-\frac{1}{b}\tilde{\mJ}^\top
\bigl(\mK+b\lambda\mI_b\bigr)\vu
=
-\frac{1}{b}\tilde{\mJ}^\top\tilde{\vr}
=
-\vg.
\]
Thus solving the whitened row-space system and backprojecting 
gives the exact damped FGN direction. 
Appendix~\ref{apx:sec4_equivalence} gives the full
parameter-space / row-space equivalence proof.

\subsection{Matrix-free operator and cost}
\label{sec:fgn_complexity}

CG only requires products with the whitened row-space operator
\begin{align}
\mB\vv
&\coloneqq
\bigl(\mK+b\,\lambda\mI_b\bigr)\vv
=
\tilde{\mJ}\bigl(\tilde{\mJ}^\top\vv\bigr)
+
b\,\lambda\vv,
\qquad
\vv\in\R^b .
\label{eq:fgn_row_operator_whitened}
\end{align}
The matrix $\mB$ is never materialized. One application can be implemented as
\begin{align}
\va &= \mQ^{1/2}\vv,
&
\vbeta &= \mJ^\top \va,
&
\vgamma &= \mJ \vbeta,
&
\mB \vv &= \mQ^{1/2}\vgamma + b\,\lambda \vv .
\end{align}
Thus each operator application uses one vector-Jacobian product (VJP) 
and one Jacobian-vector product (JVP) 
with the batched true-vs-rest margin map, 
plus $\mathcal{O}(b)$ vector operations. 
In standard row-space CG, 
the direction~\eqref{eq:fgn_direction_from_u}
is recovered after termination by one additional VJP.

In autodiff implementations, 
the batched margin map can be linearized once at the current parameters, 
and the resulting JVP/VJP closures can be reused across CG iterations. 
With zero initialization, at most $N_{\mathrm{CG}}$ CG iterations
require at most $N_{\mathrm{CG}}$ applications
of~\eqref{eq:fgn_row_operator_whitened},
each using one JVP and one VJP,
plus $\mathcal{O}(N_{\mathrm{CG}}b)$ row-space vector operations
and, in the standard implementation, 
one final VJP for backprojection.

A lifted implementation instead accumulates the parameter direction
using the VJP already computed in each iteration, 
removing the final backprojection.
It is equivalent to standard row-space CG at every iteration (in exact arithmetic),
but requires storing and updating a $d$-dimensional vector.

FGN reduces the second-order solve 
to a $b$-dimensional row-space system
and replaces class-coupled softmax covariance products 
with scalar-margin Jacobian products and diagonal reweighting.
Full softmax GGN is also matrix-free, 
but each curvature product still
propagates through all $C$ logits and applies the softmax covariance at a
cost of $\mathcal{O}(C)$ per example.
FGN still evaluates all $C$ logits and probabilities, 
so the complete training cost and memory remain class-dependent; 
its solve dimension and persistent row-space CG state, however, 
are independent of $C$.
Algorithm~\ref{alg:fgn_core} in Appendix~\ref{apx:fgn_algorithm}
summarizes the core step, and Appendix~\ref{apx:solver_cost}
compares its computational and storage requirements
with full GGN, SGD, and Adam.

\subsection{Solver interpretation and convergence properties}
\label{sec:fgn_convergence}

Let
\[
\mA\coloneqq \mH^{\mathrm{FGN}}+\lambda \mI_{d},
\qquad
\mP\coloneqq \mA^{-1}.
\]
With an exact inner solve, \eqref{eq:fgn_damped_system_param} gives
$\vd=-\mP\vg$, so the outer update
$\vw^+=\vw+\lr\vd$ 
has the form of a stochastic preconditioned gradient step 
with the FGN preconditioner. 
The row-space system is only the computational route used to
apply $\mP$ without forming $\mA$.

With truncated CG, the same viewpoint holds with an inexact preconditioner
application. If
\[
\ve
\coloneqq
\bigl(\mK+b\lambda\mI_b\bigr)\vu-\tilde{\vr},
\qquad
\vd\coloneqq-\tilde{\mJ}^\top\vu,
\]
then the induced residual in the original damped parameter-space system obeys
\begin{align}
\bigl(\mH^{\mathrm{FGN}}+\lambda\mI_d\bigr)\vd+\vg
&=
-\frac{1}{b}\tilde{\mJ}^\top\ve,
&
\bigl\|
\bigl(\mH^{\mathrm{FGN}}+\lambda\mI_d\bigr)\vd+\vg
\bigr\|
&\le
\frac{1}{b}\|\tilde{\mJ}\|\,\|\ve\|.
\label{eq:fgn_row_to_param_residual_bound}
\end{align}
An exact row-space solve ($\ve=\vzero$) recovers the damped FGN direction,
while~\eqref{eq:fgn_row_to_param_residual_bound}
relates the row-space residual to the parameter-space residual.

Appendix~\ref{apx:sec4_convergence} provides a convergence analysis 
for a simplified preconditioned update with fixed damping
and independent gradient and curvature mini-batches.
For the empirical risk $\Risk$ defined there, assume that 
$\Risk$ is bounded below with a Lipschitz gradient,
the stochastic gradients are unbiased with bounded second moments,
the preconditioner has uniform positive lower and upper spectral bounds,
and
the positive step sizes satisfy
$\sum_t\lr_t=\infty$ and $\sum_t\lr_t^2<\infty$.
Theorem~\ref{thm:app_solver_simplified} then establishes
almost-sure convergence of $\Risk(\vw_t)$ to a finite limit and
\[
\liminf_{t\to\infty}
\E\bigl[\|\nabla\Risk(\vw_t)\|^2\bigr]=0.
\]
Lemma~\ref{lem:app_solver_inexact} extends these conclusions to
inexact solves under its relative residual and summability conditions.
The analysis follows standard stochastic preconditioning 
and
inexact-solve arguments
\citep{bottou2018optimization,li2017preconditioned,baey2023efficient,dembo1982inexact,bollapragada2019exact}.
These guarantees do not cover the shared-batch Algorithm~\ref{alg:fgn_core} 
or its practical extensions.

%%%%%%%%%%%%%%%%%%%%%%%%%%%%%%%%%%%%%%%%%%%%%%%%%%%%%%%%%%%%%%%%%%%%%%%%%%%%%%%
%%%%%%%%%%%%%%%%%%%%%%%%%%%%%%%%%%%%%%%%%%%%%%%%%%%%%%%%%%%%%%%%%%%%%%%%%%%%%%%

\section{Experiments}
\label{sec:exp}

We evaluate the FGN curvature approximation and its implementation
in three settings.
% 5.1
Section~\ref{sec:exp_mechanism_checks} examines update cost,
the output-space decomposition in
Theorem~\ref{thm:fgn_ggn_decomposition},
and agreement between damped FGN and full GGN steps.
% 5.2
Section~\ref{sec:exp_frozen_head} compares optimization
of an affine head on frozen features, 
where full softmax GGN
equals the Hessian of the loss.
% 5.3
Section~\ref{sec:exp_e2e} evaluates end-to-end ResNet-20 training 
on CIFAR-100 with the same maximum CG iteration budget for SGN and FGN.
We use SGN~\citep{gargiani2020promise} as the
full softmax GGN baseline, 
with damped steps computed using truncated CG.
% Implementation
The FGN algorithm is implemented in Somax~\citep{korbit2026second},
a JAX library for second-order optimization.
Scripts to reproduce the experimental results are available 
at \url{https://github.com/cor3bit/fgn}.
Further experimental details are provided in Appendix~\ref{apx:experiment_details}.

\subsection{Mechanism checks}
\label{sec:exp_mechanism_checks}

\begin{figure}[t]
\centering
\includegraphics[width=0.98\textwidth]{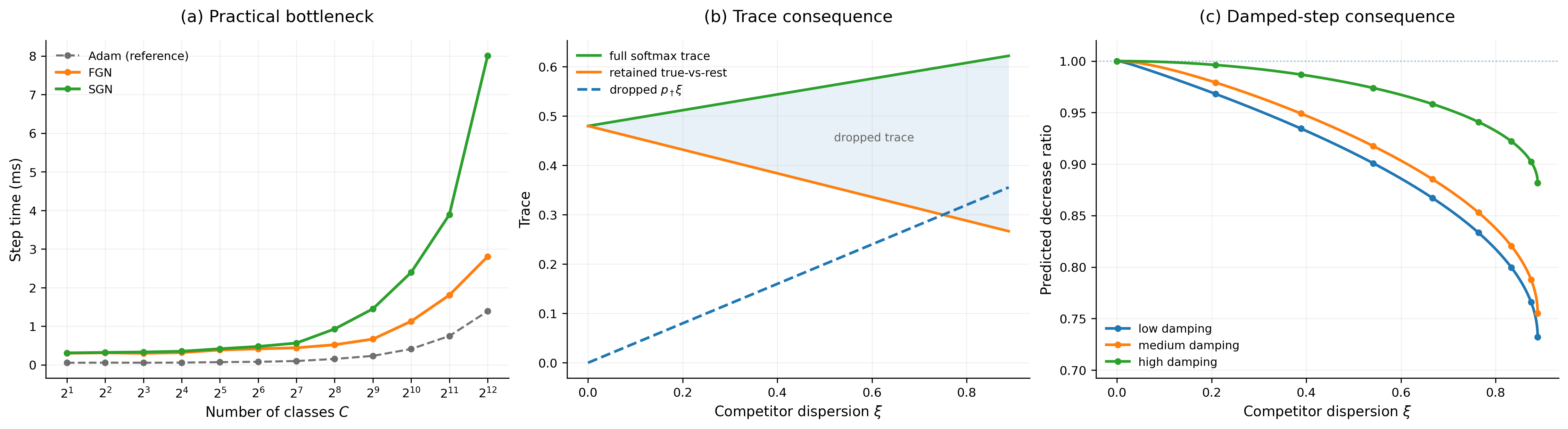}
\caption{
\textbf{Mechanism checks for FGN.}
(a) Steady-state optimizer update time for a synthetic affine head,
with at most $5$ CG iterations for SGN and FGN.
(b) Output-space traces at fixed $p_\star=0.6$; 
the shaded gap is $p_\dagger\xi$.
(c) Ratio of FGN to full GGN predicted decreases under the same damped
quadratic model in a two-dimensional construction ($p_\star=0.7$).
}
\label{fig:triptych}
\end{figure}

Panel~(a) measures the cost of the implemented updates 
for an affine softmax head on fixed $2048$-dimensional Gaussian features, 
using batch size $512$ and $C=2^k$ for $k=1,\ldots,12$.
SGN and FGN use fixed damping $\lambda=1$ and at most five CG iterations;
Adam~\citep{kingma2014adam} provides a first-order cost reference.
After compilation and warmup, 
FGN has lower update time than SGN,
with a $2.85\times$ speedup at $C=4096$.
Both methods are matrix-free:
SGN uses parameter-space CG with the full softmax curvature,
whereas FGN uses row-space CG 
with a single true-vs-rest margin per example.
Both methods evaluate all $C$ logits, and timings include
complete optimizer updates.

Panel~(b) visualizes the output-space trace decomposition in
Theorem~\ref{thm:fgn_ggn_decomposition} 
for $C=10$ and fixed $p_\star=0.6$.
The dropped trace $p_\dagger\xi$ increases linearly
with competitor dispersion.
The parameter-space gap additionally depends on the geometry
of the competitor Jacobian rows through~\eqref{eq:fgn_parameter_space_gap}.

Panel~(c) compares damped FGN and full GGN steps in a
two-dimensional construction with affine logits
($C=10$, $p_\star=0.7$).
Both steps are computed by direct solves and evaluated under the same 
full GGN damped quadratic model.
The plotted ratio measures the predicted decrease achieved
by the FGN step relative to the full GGN step.
The ratio is close to one for concentrated competitor distributions
and decreases as $\xi$ increases.
Stronger damping reduces the influence of the dropped competitor covariance
on the resulting steps, 
keeping the ratio closer to one.

\subsection{Frozen-feature multiclass head optimization}
\label{sec:exp_frozen_head}

\begin{figure*}[t]
\centering
\begin{minipage}[t]{0.49\textwidth}
    \centering
    \includegraphics[width=\linewidth]{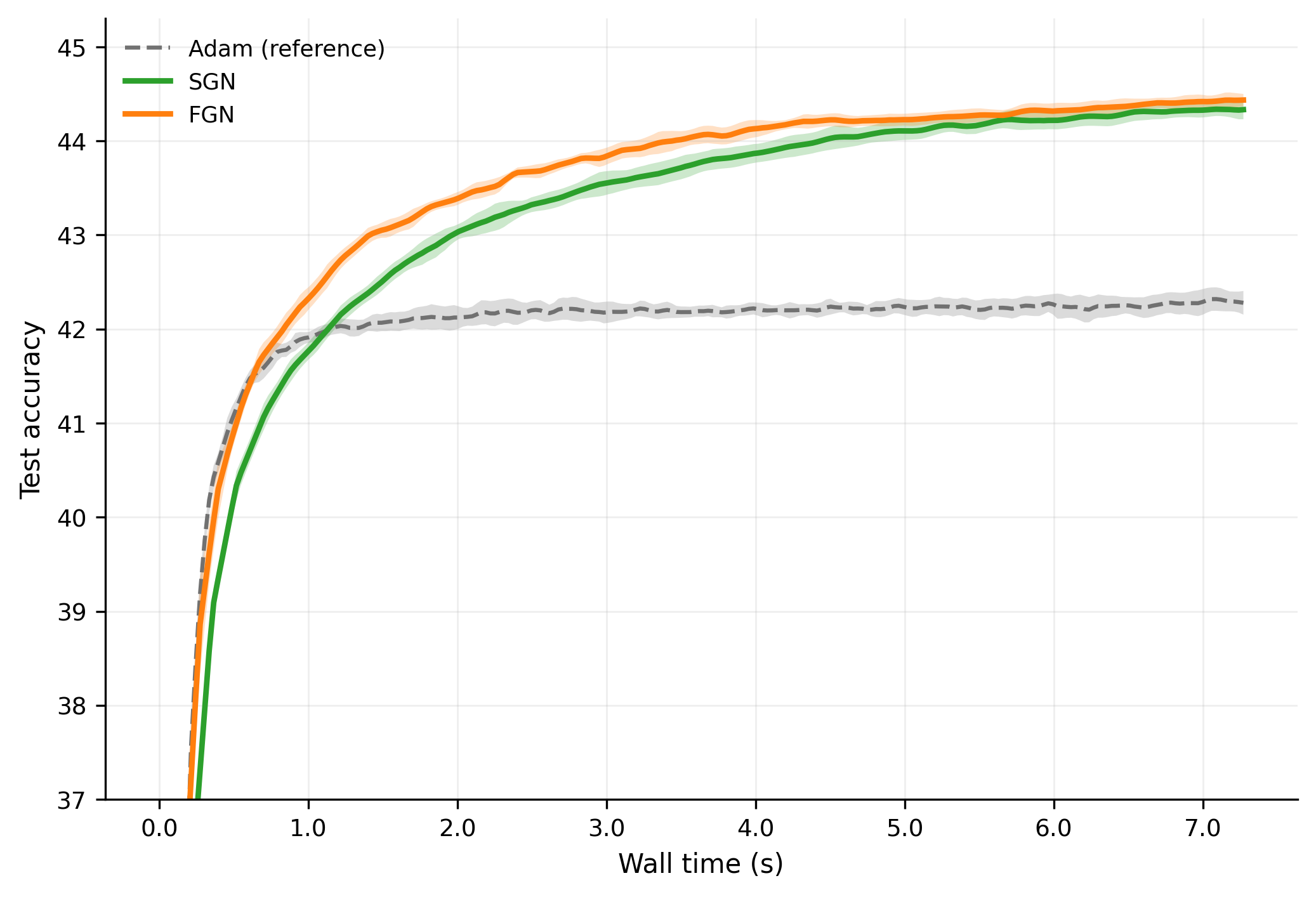}
    \vspace{-2mm}
    \centerline{\small (a) Frozen-feature head optimization}
\end{minipage}
\hfill
\begin{minipage}[t]{0.49\textwidth}
    \centering
    \includegraphics[width=\linewidth]{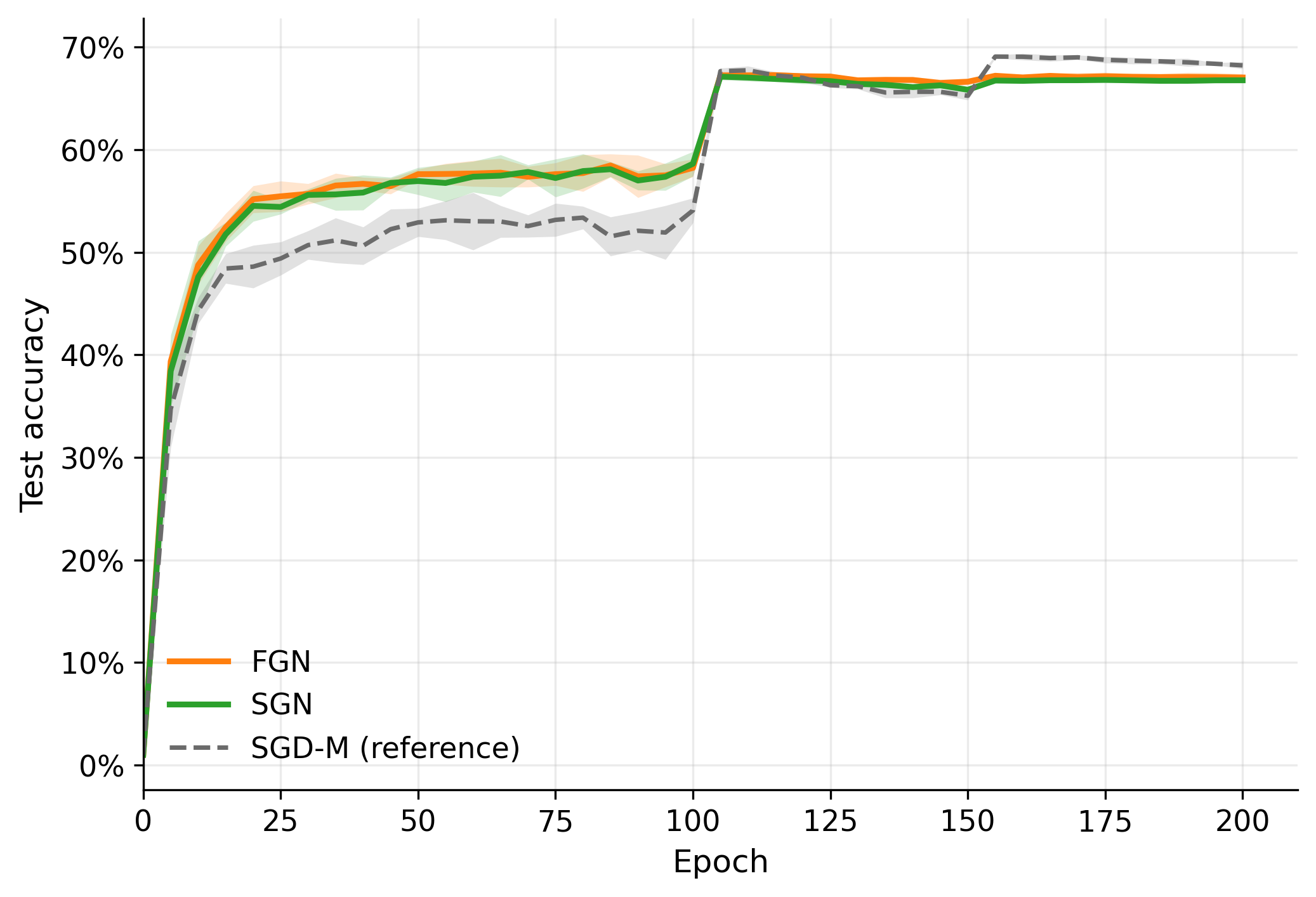}
    \vspace{-2mm}
    \centerline{\small (b) ResNet-20 on CIFAR-100}
\end{minipage}
\vspace{2mm}
\caption{
\textbf{FGN vs full GGN in multiclass classification.}
(a) Cars196 frozen-feature head: 
test accuracy against head optimization wall time, 
excluding feature extraction, compilation, and evaluation.
(b) ResNet-20 on CIFAR-100: 
test accuracy against training epoch.
Curves show means over $10$ seeds in (a) and $5$ seeds in (b);
shaded bands indicate $\pm1$ standard deviation.
}
\label{fig:headopt_and_e2e}
\end{figure*}

To compare FGN with the exact Hessian,
we train an affine softmax head on Cars196~\citep{KrauseStarkDengFei-Fei_3DRR2013},
a fine-grained $196$-class classification dataset.
We use a frozen ImageNet-pretrained ResNet-50~\citep{he2016deep}
to extract fixed $2048$-dimensional features $\vh_i$,
and train only the affine head
$\vz_i=\mW^\top\vh_i+\vb$.
Since the logits are affine in the optimized parameters,
full softmax GGN equals the Hessian of the loss.
SGN applies truncated CG in parameter space to the full GGN.
FGN uses a closed-form row-space Gram for the retained
true-vs-rest curvature.
For both methods, 
damping is fixed at $\lambda=1$ and CG is capped at $5$ iterations.

Figure~\ref{fig:headopt_and_e2e}(a) compares test accuracy
against head optimization time, excluding feature extraction,
compilation, and evaluation.
Hyperparameters are selected on a validation split
and then fixed for final test evaluation.
At the end of the shared $7.27$ s horizon,
Adam, SGN, and FGN reach
$42.28\pm0.13\%$, $44.33\pm0.10\%$, and
$44.43\pm0.07\%$ test accuracy, respectively.
FGN reaches $43\%$ and $44\%$ accuracy in $1.41$ and $3.45$ s,
compared with $1.97$ and $4.44$ s for SGN.
FGN achieves comparable test accuracy to SGN 
while reaching the same accuracy thresholds sooner.
This occurs even though the dropped term accounts for $40.85\%$
of the full parameter-space GGN trace at the final FGN iterate
(Appendix~\ref{apx:fig2_headopt_details}).

\subsection{End-to-end training}
\label{sec:exp_e2e}

We train a CIFAR-style ResNet-20~\citep{he2016deep}
end-to-end on CIFAR-100~\citep{ds_cifar10}
to test whether the FGN under-approximation
preserves optimization performance beyond the affine-head setting.
We compare FGN with SGN, which uses the full softmax GGN,
and include SGD with momentum (SGD-M) as a first-order reference.
Unlike the frozen-feature head, 
the logits are now nonlinear in the optimized parameters.

SGD-M follows a standard CIFAR ResNet training recipe
with $200$ epochs, batch size $128$, and an initial learning
rate of $0.1$, reduced by $10$ at epochs $100$ and $150$.
SGN and FGN follow the same decay schedule, 
starting from learning rate $0.3$.
Both curvature methods use fixed damping $\lambda=1$,
at most $5$ CG iterations, and tolerance $10^{-5}$,
with limited hyperparameter tuning.

Figure~\ref{fig:headopt_and_e2e}(b) shows similar
test accuracy trajectories for FGN and SGN.
After $200$ epochs, FGN reaches $67.05\pm0.35\%$ test accuracy,
compared with $66.78\pm0.23\%$ for SGN.
The corresponding training step times are $76.15$ and $95.50$ ms,
giving FGN a $1.25\times$ speedup.
SGD-M is faster and achieves higher final test accuracy
($68.23\pm0.34\%$).
SGN reaches lower training cross-entropy, while FGN has lower
test cross-entropy (Appendix~\ref{apx:e2e}, Table~\ref{tab:cifar100_main}).

Thus, with the same CG iteration cap, 
FGN reaches similar
test accuracy to SGN at lower update cost during nonlinear
end-to-end training, despite discarding the within-competitor covariance.

%%%%%%%%%%%%%%%%%%%%%%%%%%%%%%%%%%%%%%%%%%%%%%%%%%%%%%%%%%%%%%%%%%%%%%%%%%%%%%%
%%%%%%%%%%%%%%%%%%%%%%%%%%%%%%%%%%%%%%%%%%%%%%%%%%%%%%%%%%%%%%%%%%%%%%%%%%%%%%%

\section{Conclusion}
\label{sec:outro}

We introduced Fast Gauss-Newton (FGN),
a structured approximation to multiclass softmax GGN.
FGN is based on an exact decomposition of GGN
into a retained true-vs-rest term 
and a PSD within-competitor covariance residual. 
It preserves the original softmax loss and gradient,
yields a batch row-space system solvable by matrix-free conjugate gradient, 
and recovers the full softmax GGN curvature in the binary case.

The approximation is regime-dependent: 
it is tighter when competitor mass is concentrated and can deviate more from full GGN when probability mass is spread across several competitors.
Damping does not change this curvature gap,
but can reduce its effect on the resulting step.

In both frozen-feature head optimization and 
end-to-end ResNet-20 training on CIFAR-100, 
FGN achieves test accuracy similar to the full GGN baseline with faster updates.
A natural next step is to extend the true-vs-rest model 
with low-rank approximations of the within-competitor covariance.

%%%%%%%%%%%%%%%%%%%%%%%%%%%%%%%%%%%%%%%%%%%%%%%%%%%%%%%%%%%%%%%%%%%%%%%%%%%%%%%
%%%%%%%%%%%%%%%%%%%%%%%%%%%%%%%%%%%%%%%%%%%%%%%%%%%%%%%%%%%%%%%%%%%%%%%%%%%%%%%

\begin{ack}

This work was funded by the Italian Ministero dell'Università e della Ricerca, project PRIN ``Too much Information! Control-Oriented, Dynamic Information Selection'' (Grant: 2022NLTTBN).

The authors have no competing interests to declare 
that are relevant to the content of this paper.

\end{ack}

%%%%%%%%%%%%%%%%%%%%%%%%%%%%%%%%%%%%%%%%%%%%%%%%%%%%%%%%%%%%%%%%%%%%%%%%%%%%%%%
%%%%%%%%%%%%%%%%%%%%%%%%%%%%%%%%%%%%%%%%%%%%%%%%%%%%%%%%%%%%%%%%%%%%%%%%%%%%%%%

\bibliographystyle{plainnat}

\bibliography{references}

@book{nocedal2006numerical,
  title={Numerical Optimization},
  author={Nocedal, J. and Wright, S.},
  isbn={9780387400655},
  lccn={2006923897},
  series={Springer Series in Operations Research and Financial Engineering},
  year={2006},
  publisher={Springer New York}
}

@incollection{robbins1971convergence,
  title={A convergence theorem for non negative almost supermartingales and some applications},
  author={Robbins, Herbert and Siegmund, David},
  booktitle={Optimizing methods in statistics},
  pages={233--257},
  year={1971},
  publisher={Elsevier}
}

@article{kingma2014adam,
  title={Adam: A method for stochastic optimization},
  author={Kingma, Diederik P and Ba, Jimmy},
  journal={arXiv preprint arXiv:1412.6980},
  year={2014}
}

@article{bottou2018optimization,
  title={Optimization methods for large-scale machine learning},
  author={Bottou, L{\'e}on and Curtis, Frank E and Nocedal, Jorge},
  journal={SIAM review},
  volume={60},
  number={2},
  pages={223--311},
  year={2018},
  publisher={SIAM}
}

@inproceedings{botev2017practical,
  title={Practical {G}auss-{N}ewton optimisation for deep learning},
  author={Botev, Aleksandar and Ritter, Hippolyt and Barber, David},
  booktitle={International Conference on Machine Learning},
  pages={557--565},
  year={2017},
  organization={PMLR}
}

@article{gargiani2020promise,
  title={On the Promise of the Stochastic Generalized {G}auss-{N}ewton Method for Training {DNNs}},
  author={Gargiani, Matilde and Zanelli, Andrea and Diehl, Moritz and Hutter, Frank},
  journal={arXiv preprint arXiv:2006.02409},
  year={2020}
}

@article{sagun2017empirical,
	title={Empirical analysis of the hessian of over-parametrized neural networks},
	author={Sagun, Levent and Evci, Utku and Guney, V Ugur and Dauphin, Yann and Bottou, Leon},
	journal={arXiv preprint arXiv:1706.04454},
	year={2017}
}

@article{ren2019efficient,
  title={Efficient subsampled Gauss-Newton and natural gradient methods for training neural networks},
  author={Ren, Yi and Goldfarb, Donald},
  journal={arXiv preprint arXiv:1906.02353},
  year={2019}
}

@inproceedings{martens2010deep,
  title={Deep learning via {H}essian-free optimization.},
  author={Martens, James and others},
  booktitle={ICML},
  volume={27},
  pages={735--742},
  year={2010}
}

@inproceedings{martens2011learning,
  title={Learning recurrent neural networks with {H}essian-free optimization},
  author={Martens, James and Sutskever, Ilya},
  booktitle={Proceedings of the 28th international conference on machine learning (ICML-11)},
  pages={1033--1040},
  year={2011}
}

@article{martens2020new,
  title={New insights and perspectives on the natural gradient method},
  author={Martens, James},
  journal={Journal of Machine Learning Research},
  volume={21},
  number={146},
  pages={1--76},
  year={2020}
}

@inproceedings{wiesler2013investigations,
  title={Investigations on {H}essian-free optimization for cross-entropy training of deep neural networks.},
  author={Wiesler, Simon and Li, Jinyu and Xue, Jian},
  booktitle={Interspeech},
  pages={3317--3321},
  year={2013}
}

@inproceedings{martens2015optimizing,
  title={Optimizing neural networks with kronecker-factored approximate curvature},
  author={Martens, James and Grosse, Roger},
  booktitle={International conference on machine learning},
  pages={2408--2417},
  year={2015},
  organization={PMLR}
}

@article{kiros2013training,
  title={Training neural networks with stochastic {H}essian-free optimization},
  author={Kiros, Ryan},
  journal={arXiv preprint arXiv:1301.3641},
  year={2013}
}

@article{bollapragada2019exact,
  title={Exact and inexact subsampled {N}ewton methods for optimization},
  author={Bollapragada, Raghu and Byrd, Richard H and Nocedal, Jorge},
  journal={IMA Journal of Numerical Analysis},
  volume={39},
  number={2},
  pages={545--578},
  year={2019},
  publisher={Oxford University Press}
}

@article{dembo1982inexact,
  title={Inexact newton methods},
  author={Dembo, Ron S and Eisenstat, Stanley C and Steihaug, Trond},
  journal={SIAM Journal on Numerical analysis},
  volume={19},
  number={2},
  pages={400--408},
  year={1982},
  publisher={SIAM}
}

@article{schraudolph2002fast,
  title={Fast curvature matrix-vector products for second-order gradient descent},
  author={Schraudolph, Nicol N},
  journal={Neural computation},
  volume={14},
  number={7},
  pages={1723--1738},
  year={2002},
  publisher={MIT Press}
}

@inproceedings{sankar2021deeper,
  title={A deeper look at the {H}essian eigenspectrum of deep neural networks and its applications to regularization},
  author={Sankar, Adepu Ravi and Khasbage, Yash and Vigneswaran, Rahul and Balasubramanian, Vineeth N},
  booktitle={Proceedings of the AAAI Conference on Artificial Intelligence},
  volume={35},
  number={11},
  pages={9481--9488},
  year={2021}
}

@article{papyan2018full,
  title={The full spectrum of deep net {H}essians at scale: Dynamics with sample size},
  author={Papyan, Vardan},
  journal={arXiv preprint arXiv:1811.07062},
  year={2018}
}

@article{amari1998natural,
  title={Natural gradient works efficiently in learning},
  author={Amari, Shun-Ichi},
  journal={Neural computation},
  volume={10},
  number={2},
  pages={251--276},
  year={1998},
  publisher={MIT Press}
}

@article{liu2023sophia,
  title={Sophia: A Scalable Stochastic Second-order Optimizer for Language Model Pre-training},
  author={Liu, Hong and Li, Zhiyuan and Hall, David and Liang, Percy and Ma, Tengyu},
  journal={arXiv preprint arXiv:2305.14342},
  year={2023}
}

@inproceedings{yao2021adahessian,
  title={Adahessian: An adaptive second order optimizer for machine learning},
  author={Yao, Zhewei and Gholami, Amir and Shen, Sheng and Mustafa, Mustafa and Keutzer, Kurt and Mahoney, Michael},
  booktitle={proceedings of the AAAI conference on artificial intelligence},
  volume={35},
  number={12},
  pages={10665--10673},
  year={2021}
}

@article{ds_cifar10,
  title={Learning multiple layers of features from tiny images},
  author={Krizhevsky, Alex and Hinton, Geoffrey and others},
  year={2009},
  publisher={Toronto, ON, Canada}
}

@inproceedings{KrauseStarkDengFei-Fei_3DRR2013,
title = {3D Object Representations for Fine-Grained Categorization},
booktitle = {4th International IEEE Workshop on  3D Representation and Recognition (3dRR-13)},
year = {2013},
address = {Sydney, Australia},
author = {Jonathan Krause and Michael Stark and Jia Deng and Li Fei-Fei}
}

@software{jax2018github,
  author = {James Bradbury and Roy Frostig and Peter Hawkins and Matthew James Johnson and Yash Katariya and Chris Leary and Dougal Maclaurin and George Necula and Adam Paszke and Jake Vander{P}las and Skye Wanderman-{M}ilne and Qiao Zhang},
  title = {{JAX}: composable transformations of {P}ython+{N}um{P}y programs},
  url = {http://github.com/jax-ml/jax},
  version = {0.3.13},
  year = {2018},
}

@inproceedings{he2016deep,
  title={Deep residual learning for image recognition},
  author={He, Kaiming and Zhang, Xiangyu and Ren, Shaoqing and Sun, Jian},
  booktitle={Proceedings of the IEEE conference on computer vision and pattern recognition},
  pages={770--778},
  year={2016}
}

@article{korbit2025exact,
  title={Exact gauss-newton optimization for training deep neural networks},
  author={Korbit, Mikalai and Adeoye, Adeyemi D and Bemporad, Alberto and Zanon, Mario},
  journal={Neurocomputing},
  pages={131738},
  year={2025},
  publisher={Elsevier}
}

@inproceedings{korbit2026second,
  title={Second-Order, First-Class: A Composable Stack for Curvature-Aware Training},
  author={Korbit, Mikalai and Zanon, Mario},
  booktitle={Joint European Conference on Machine Learning and Knowledge Discovery in Databases},
  pages={90--105},
  year={2026},
  organization={Springer}
}

@inproceedings{gupta2018shampoo,
  title={Shampoo: Preconditioned stochastic tensor optimization},
  author={Gupta, Vineet and Koren, Tomer and Singer, Yoram},
  booktitle={International Conference on Machine Learning},
  pages={1842--1850},
  year={2018},
  organization={PMLR}
}

@article{pearlmutter1994fast,
  title={Fast exact multiplication by the Hessian},
  author={Pearlmutter, Barak A},
  journal={Neural computation},
  volume={6},
  number={1},
  pages={147--160},
  year={1994},
  publisher={MIT Press}
}

@article{li2017preconditioned,
  title={Preconditioned stochastic gradient descent},
  author={Li, Xi-Lin},
  journal={IEEE transactions on neural networks and learning systems},
  volume={29},
  number={5},
  pages={1454--1466},
  year={2017},
  publisher={IEEE}
}

@inproceedings{baey2023efficient,
  title={Efficient preconditioned stochastic gradient descent for estimation in latent variable models},
  author={Baey, Charlotte and Delattre, Maud and Kuhn, Estelle and Leger, Jean-Benoist and Lemler, Sarah},
  booktitle={International Conference on Machine Learning},
  pages={1430--1453},
  year={2023},
  organization={PMLR}
}

@article{liao2021hessian,
  title={Hessian eigenspectra of more realistic nonlinear models},
  author={Liao, Zhenyu and Mahoney, Michael W},
  journal={Advances in Neural Information Processing Systems},
  volume={34},
  pages={20104--20117},
  year={2021}
}

@misc{xie2022power,
      title={On the Power-Law Hessian Spectrums in Deep Learning}, 
      author={Zeke Xie and Qian-Yuan Tang and Yunfeng Cai and Mingming Sun and Ping Li},
      year={2022},
      eprint={2201.13011},
      archivePrefix={arXiv},
      primaryClass={cs.LG},
      url={https://arxiv.org/abs/2201.13011}, 
}

@article{hestenes1952methods,
  title={Methods of conjugate gradients for solving linear systems},
  author={Hestenes, Magnus R and Stiefel, Eduard and others},
  journal={Journal of research of the National Bureau of Standards},
  volume={49},
  number={6},
  pages={409--436},
  year={1952}
}

@inproceedings{chen2016strategies,
  title={Strategies for training large vocabulary neural language models},
  author={Chen, Wenlin and Grangier, David and Auli, Michael},
  booktitle={Proceedings of the 54th Annual Meeting of the Association for Computational Linguistics (Volume 1: Long Papers)},
  pages={1975--1985},
  year={2016}
}

@inproceedings{daghaghi2021tale,
  title={A tale of two efficient and informative negative sampling distributions},
  author={Daghaghi, Shabnam and Medini, Tharun and Meisburger, Nicholas and Chen, Beidi and Zhao, Mengnan and Shrivastava, Anshumali},
  booktitle={International conference on machine learning},
  pages={2319--2329},
  year={2021},
  organization={PMLR}
}

@article{allwein2000reducing,
  title={Reducing multiclass to binary: A unifying approach for margin classifiers},
  author={Allwein, Erin L and Schapire, Robert E and Singer, Yoram},
  journal={Journal of machine learning research},
  volume={1},
  number={Dec},
  pages={113--141},
  year={2000}
}

@article{rifkin2004defense,
  title={In defense of one-vs-all classification},
  author={Rifkin, Ryan and Klautau, Aldebaro},
  journal={Journal of machine learning research},
  volume={5},
  number={Jan},
  pages={101--141},
  year={2004}
}

@inproceedings{bengio2003quick,
  title={Quick training of probabilistic neural nets by importance sampling},
  author={Bengio, Yoshua and Sen{\'e}cal, Jean-S{\'e}bastien},
  booktitle={International Workshop on Artificial Intelligence and Statistics},
  pages={17--24},
  year={2003},
  organization={PMLR}
}

@inproceedings{kobayashi2019large,
  title={Large Margin In Softmax Cross-Entropy Loss.},
  author={Kobayashi, Takumi},
  booktitle={BMVC},
  volume={2},
  pages={3},
  year={2019}
}

%%%%%%%%%%%%%%%%%%%%%%%%%%%%%%%%%%%%%%%%%%%%%%%%%%%%%%%%%%%%

\clearpage
\appendix

\section{Related Work}\label{apx:related_work}

This appendix reviews prior work 
around the paper's central technical choice:
FGN modifies the \emph{output-space} curvature model for softmax cross-entropy 
(via an exact true-vs-rest margin reformulation 
and a decomposition of the softmax GGN), 
and only then builds a row-space solver around the resulting batched form
$\mH^{\mathrm{FGN}}(\vw;\calB)=\frac{1}{b}\mJ^\top \mQ \mJ$.

\paragraph{Gauss-Newton, generalized Gauss-Newton, and Fisher curvature.}
The Gauss-Newton and generalized Gauss-Newton matrices 
are standard PSD curvature approximations to the Hessian, 
obtained by linearizing the model and retaining only the curvature
of the loss in \emph{output space}~\citep{schraudolph2002fast,bottou2018optimization}.
For log-likelihood losses in exponential-family models, 
this output-space curvature is closely
related to the Fisher information matrix, 
which also underlies natural gradient methods in
information geometry~\citep{amari1998natural,martens2020new}.
This Fisher/GGN perspective is central to scalable second-order 
optimizers in deep learning.
It motivates replacing the indefinite Hessian with a PSD curvature approximation
and 
explains the common use of damping or trust-region style regularization 
for robustness in practice
(e.g., via $(\mH+\lambda \mI)\vd=-\vg$)~\citep{martens2020new,nocedal2006numerical}.
Whereas much prior work retains the standard softmax output-space curvature 
and focuses on approximating or inverting the resulting parameter-space matrix,
FGN modifies the output-space curvature factor itself 
while keeping the training loss and gradient exact.

\paragraph{Softmax cross-entropy curvature and the output-space bottleneck.}
For multiclass softmax cross-entropy, 
the canonical logit-space curvature factor is the
$C\times C$ softmax covariance $\diag(\vp)-\vp\vp^\top$, 
leading to a per-example GGN term
$\mJ_z^\top(\diag(\vp)-\vp\vp^\top)\mJ_z$~\citep{schraudolph2002fast,martens2020new}.
Matrix-free GGN products avoid forming the parameter-space matrix,
but still apply the softmax covariance to $C$-dimensional vectors.
The core contribution of this paper is an \emph{output-space} simplification 
that is specific to softmax cross-entropy: 
by reformulating the loss exactly as a scalar softplus of a true-vs-rest
margin and then decomposing the softmax GGN, 
we isolate a true-vs-rest margin term and a PSD covariance
term over competing logits. FGN retains the former and discards the latter. 
This differs from diagonal or structured approximations 
that keep the original softmax curvature model but compress the
resulting parameter-space curvature.

\paragraph{Output approximations and multiclass reductions.}
Large-output models are often scaled by changing the training computation
itself, 
for example through sampled or approximate output losses such as
importance sampling, sampled softmax, target sampling, or negative
sampling~\citep{bengio2003quick,chen2016strategies,daghaghi2021tale}.
Another approach reduces multiclass classification to collections of binary
problems, including one-vs-rest schemes~\citep{allwein2000reducing,rifkin2004defense}.
FGN is different from both families: it keeps the original softmax
cross-entropy loss and exact gradient, and approximates only the GGN curvature
factor.

\paragraph{Second-order methods in deep learning.}
A large literature studies second-order training methods whose main computational challenge is
building and inverting curvature approximations.
Hessian-free optimization and related truncated-Newton methods solve 
linear systems using
conjugate gradients and matrix-free curvature-vector products 
computed by automatic
differentiation~\citep{martens2010deep,martens2011learning,pearlmutter1994fast,hestenes1952methods}.
These methods typically use the exact Hessian or GN/GGN curvature 
and introduce damping to stabilize the solve.
Structured factorization approaches aim to retain richer geometry at lower cost.
K-FAC approximates the Fisher/GGN by layerwise Kronecker factors~\citep{martens2015optimizing},
while Shampoo exploits the tensor structure of parameter blocks 
and maintains separate preconditioners along each axis~\citep{gupta2018shampoo}.
Recent diagonal curvature methods such as AdaHessian and Sophia estimate 
(variants of) diagonal
Hessian information at low overhead~\citep{yao2021adahessian,liu2023sophia}.
FGN is complementary to these methods:
rather than introducing another layerwise factorization 
of the same softmax curvature,
it modifies the \emph{multiclass output-space} curvature model itself,
yielding a different batched Jacobian structure before the linear solve.

\paragraph{Stochastic Gauss-Newton and observation/row-space solvers.}
Practical deep learning GN methods build tractable mini-batch GN/GGN approximations~\citep{botev2017practical}.
Given such approximations, stochastic GN methods typically solve the resulting
damped linear systems only approximately, often by CG or related stochastic
linear algebra procedures~\citep{gargiani2020promise}.
A separate line of work exploits low-rank structure and observation-space
identities to obtain solves
that scale with batch size rather than parameter dimension
\citep{ren2019efficient,korbit2025exact}.
These methods demonstrate that row-space solvers can be
effective, but they typically retain the full multiclass output geometry in the
curvature model. 
FGN targets the orthogonal bottleneck: the true-vs-rest
decomposition gives one scalar-margin row per example, 
so the FGN row-space
system is $b$-dimensional rather than tied 
to the full $bC$ logit block of the multiclass curvature.

\paragraph{Empirical and theoretical structure of Hessian spectra.}
Empirical studies of deep network Hessians report 
a characteristic
spiked-plus-bulk structure, 
with a small number of outlier eigenvalues separated
from a near-zero bulk~\citep{papyan2018full,sagun2017empirical}. 
Related work
further studies layerwise and power-law spectral 
structure~\citep{sankar2021deeper,xie2022power,liao2021hessian}. 
These results motivate
structured curvature models broadly, 
but they are not assumptions behind FGN.
The FGN decomposition is instead loss-specific and algebraic: 
it separates a
true-vs-rest margin direction from within-competitor covariance 
in logit space.

%%%%%%%%%%%%%%%%%%%%%%%%%%%%%%%%%%%%%%%%%%%%%%%%%%%%%%%%%%%%%%%%%%%%%%%%%%%%%%%
%%%%%%%%%%%%%%%%%%%%%%%%%%%%%%%%%%%%%%%%%%%%%%%%%%%%%%%%%%%%%%%%%%%%%%%%%%%%%%%

\section{Derivations for the True-vs-Rest Margin Representation}
\label{apx:margin_derivations}

This appendix verifies the identities used in Section~\ref{sec:margin}, 
namely
\eqref{eq:fgn_ce_loss}, \eqref{eq:fgn_phi_derivs},
and \eqref{eq:fgn_margin_logit_grad}.

\subsection{Exact scalar-margin representation}
\label{apx:exact_scalar_margin_form}

Starting from the single-example cross-entropy loss,
\begin{align}
\ell(\vw;\vx,\vy)
&=
-\log p_\star(\vw)
=
-\log \frac{e^{z_\star}}{\sum_{j=1}^C e^{z_j}},
\end{align}
split the denominator into the true-class term and the competing terms:
\begin{align}
\ell(\vw;\vx,\vy)
&=
-\log \frac{e^{z_\star}}{e^{z_\star}+\sum_{j\neq\star} e^{z_j}}
=
-z_\star + \log\Bigl(e^{z_\star}+\sum_{j\neq\star} e^{z_j}\Bigr).
\end{align}
Using $e^{z_\dagger}=\sum_{j\neq\star} e^{z_j}$, 
which follows from the definition of $z_\dagger$, we obtain
\begin{align}
\ell(\vw;\vx,\vy)
&=
-z_\star + \log\bigl(e^{z_\star}+e^{z_\dagger}\bigr)
=
-z_\star + \log\Bigl(e^{z_\star}\bigl(1+e^{z_\dagger-z_\star}\bigr)\Bigr)
\nonumber\\
&=
-z_\star + z_\star + \log\bigl(1+e^{z_\dagger-z_\star}\bigr)
=
\log\bigl(1+e^{z_\dagger-z_\star}\bigr).
\end{align}
Hence, with
$s(\vw)=z_\dagger(\vw)-z_\star(\vw)$
and
$\phi(s)=\log(1+e^s)$,
the loss has the exact scalar form
\begin{align}
\ell(\vw;\vx,\vy)=\phi\bigl(s(\vw)\bigr).
\end{align}

\subsection{Derivatives of the scalar link}
\label{apx:scalar_link_derivatives}

Let
\begin{align}
\phi(s)=\log(1+e^s).
\end{align}
Then
\begin{align}
\phi'(s)
&=
\frac{e^s}{1+e^s},
&
\phi''(s)
&=
\frac{e^s}{(1+e^s)^2}.
\end{align}
Substituting $s=z_\dagger-z_\star$ yields
\begin{align}
\phi'(s)
&=
\frac{e^{z_\dagger-z_\star}}{1+e^{z_\dagger-z_\star}}
=
\frac{e^{z_\dagger}}{e^{z_\star}+e^{z_\dagger}}
=
p_\dagger,
\\
\phi''(s)
&=
\frac{e^{z_\dagger-z_\star}}{\bigl(1+e^{z_\dagger-z_\star}\bigr)^2}
=
\frac{e^{z_\star}e^{z_\dagger}}{\bigl(e^{z_\star}+e^{z_\dagger}\bigr)^2}
=
p_\star p_\dagger.
\end{align}

\subsection{Logit-space derivatives of $z_\dagger$ and $s$}
\label{apx:margin_logit_derivatives}

For convenience, order the logit coordinates as
$(z_\star,\vz_{-\star})$,
with $\vrho$ as defined in \eqref{eq:fgn_rho_def}.
Since
$z_\dagger=\log\sum_{j\neq\star} e^{z_j}$
does not depend on $z_\star$, we have
\begin{align}
\frac{\partial z_\dagger}{\partial z_\star}=0.
\end{align}
For $k\neq \star$,
\begin{align}
\frac{\partial z_\dagger}{\partial z_k}
&=
\frac{e^{z_k}}{\sum_{j\neq\star} e^{z_j}}
=
\frac{p_k}{p_\dagger}
=
\rho_k.
\end{align}
Therefore
\begin{align}
\nabla_{\vz} z_\dagger
&=
\begin{bmatrix}
0\\
\vrho
\end{bmatrix},
\qquad
\nabla_{\vz} s
=
\nabla_{\vz}(z_\dagger-z_\star)
=
\begin{bmatrix}
-1\\
\vrho
\end{bmatrix}.
\end{align}
By the chain rule,
\begin{align}
\nabla_{\vw} z_\dagger
&=
\sum_{j\neq\star}\rho_j\,\nabla_{\vw} z_j,
&
\nabla_{\vw} s
&=
-\nabla_{\vw} z_\star
+
\sum_{j\neq\star}\rho_j\,\nabla_{\vw} z_j.
\end{align}
Equivalently, the margin Jacobian row is
\begin{align}
\mJ_s
=
\nabla_{\vw}s^\top
=
-\nabla_{\vw} z_\star^\top
+
\sum_{j\neq\star}\rho_j\,\nabla_{\vw} z_j^\top.
\end{align}

\subsection{Equivalence of softmax and margin gradients}
\label{apx:softmax_margin_grad}

We show that the standard softmax cross-entropy gradient
\begin{align}
\nabla_{\vw}\ell(\vw;\vx,\vy)
=
\mJ_z(\vw)^\top\bigl(\vp(\vw)-\vy\bigr)
\end{align}
coincides exactly with the margin form
\begin{align}
\nabla_{\vw}\ell(\vw;\vx,\vy)
=
p_\dagger(\vw)\,\nabla_{\vw}s(\vw).
\end{align}

Using the reordered coordinates from Appendix~\ref{apx:margin_logit_derivatives},
the target and probability vectors have the block forms
\begin{align}
\vy
&=
\begin{bmatrix}
1\\
\vzero
\end{bmatrix},
\qquad
\vp
=
\begin{bmatrix}
p_\star\\
p_\dagger \vrho
\end{bmatrix}.
\end{align}

Therefore
\begin{align}
\vp-\vy
=
\begin{bmatrix}
p_\star-1\\
p_\dagger\vrho
\end{bmatrix}
=
p_\dagger
\begin{bmatrix}
-1\\
\vrho
\end{bmatrix}
=
p_\dagger \nabla_{\vz} s.
\label{eq:app_logit_grad_identity}
\end{align}
Applying the chain rule gives
\begin{align}
\nabla_{\vw}\ell(\vw;\vx,\vy)
&=
\phi'(s)\,\nabla_{\vw}s
=
p_\dagger\,\nabla_{\vw}s
=
p_\dagger\,\mJ_s^\top
=
p_\dagger\,\mJ_z^\top \nabla_{\vz} s
=
\mJ_z^\top(\vp-\vy),
\end{align}
where the last equality uses \eqref{eq:app_logit_grad_identity}. 
Equivalently,
\begin{align}
\nabla_{\vw}\ell(\vw;\vx,\vy)
=
p_\dagger\bigl(\nabla_{\vw} z_\dagger(\vw)-\nabla_{\vw} z_\star(\vw)\bigr).
\end{align}
Thus the scalar-margin representation preserves the softmax cross-entropy 
value and gradient exactly. 
The approximation introduced by FGN 
begins only at the Gauss-Newton curvature level.

%%%%%%%%%%%%%%%%%%%%%%%%%%%%%%%%%%%%%%%%%%%%%%%%%%%%%%%%%%%%%%%%%%%%%%%%%%%%%%%
%%%%%%%%%%%%%%%%%%%%%%%%%%%%%%%%%%%%%%%%%%%%%%%%%%%%%%%%%%%%%%%%%%%%%%%%%%%%%%%

\section{Derivations for the FGN Curvature}
\label{apx:fgn}

This appendix collects the derivations and proofs deferred from
Section~\ref{sec:fgn}.
All identities below use the reordered logit coordinates
$(z_\star,\vz_{-\star})$
introduced in Section~\ref{sec:margin}.

\subsection{Two exact Hessian decompositions}
\label{apx:fgn_hessian_decomp}

The exact Hessian can be written in two equivalent ways, corresponding
to the multiclass logit map $\vz(\vw)$ and the scalar-margin map $s(\vw)$.

First, view the loss as the composition
$\ell(\vw;\vx,\vy)=L(\vz(\vw))$
with
$L(\vz)=-\log p_\star$.
Since
$\nabla_{\vz} L = \vp-\vy$
and
$\nabla_{\vz}^2 L = \diag(\vp)-\vp\vp^\top$,
the chain rule gives
\begin{align}
\nabla_{\vw}^2 \ell
&=
\mJ_z^\top\bigl(\diag(\vp)-\vp\vp^\top\bigr)\mJ_z
+
\sum_{c=1}^C (p_c-y_c)\,\nabla_{\vw}^2 z_c.
\end{align}

Second, use the exact scalar form from Section~\ref{sec:margin},
$\ell(\vw;\vx,\vy)=\phi(s(\vw))$,
with
$\phi'(s)=p_\dagger$
and
$\phi''(s)=p_\star p_\dagger$.
Applying the scalar chain rule gives
\begin{align}
\nabla_{\vw}^2 \ell
&=
\phi''(s)\,\nabla_{\vw}s\,\nabla_{\vw}s^\top
+
\phi'(s)\,\nabla_{\vw}^2 s
=
p_\star p_\dagger\,\mJ_s^\top \mJ_s
+
p_\dagger\,\nabla_{\vw}^2 s.
\end{align}

These identities clarify the hierarchy
\[
\text{exact Hessian}
\;\longrightarrow\;
\text{multiclass GGN}
\;\longrightarrow\;
\text{FGN}.
\]
The first arrow discards second derivatives of the logit map.
The second arrow additionally removes part of the output-space curvature,
as characterized in Theorem~\ref{thm:fgn_ggn_decomposition}.

\subsection{Proof of Theorem~\ref{thm:fgn_ggn_decomposition}}
\label{apx:fgn_decomposition}

Recall the block form
\begin{align*}
\vp
&=
\begin{bmatrix}
p_\star \\
\vp_{-\star}
\end{bmatrix},
&
\vp_{-\star}
&=
p_\dagger \vrho.
\end{align*}
Expanding the softmax covariance in blocks gives
\begin{align}
\diag(\vp)-\vp\vp^\top
=
\begin{bmatrix}
p_\star-p_\star^2 & -p_\star \vp_{-\star}^\top \\
-p_\star \vp_{-\star} & \diag(\vp_{-\star})-\vp_{-\star}\vp_{-\star}^\top
\end{bmatrix}
=
\begin{bmatrix}
p_\star p_\dagger & -p_\star p_\dagger\,\vrho^\top \\
-p_\star p_\dagger\,\vrho & p_\dagger\diag(\vrho)-p_\dagger^2\vrho\vrho^\top
\end{bmatrix}.
\label{eq:apx_fgn_softmax_block}
\end{align}
Using
$p_\dagger-p_\dagger^2 = p_\star p_\dagger$,
the lower-right block becomes
\begin{align}
p_\dagger\diag(\vrho)-p_\dagger^2\vrho\vrho^\top
&=
p_\star p_\dagger\,\vrho\vrho^\top
+
p_\dagger\bigl(\diag(\vrho)-\vrho\vrho^\top\bigr).
\label{eq:apx_fgn_lower_right}
\end{align}
Substituting \eqref{eq:apx_fgn_lower_right} into
\eqref{eq:apx_fgn_softmax_block} yields
\begin{align}
\diag(\vp)-\vp\vp^\top
&=
p_\star p_\dagger
\begin{bmatrix}
1 & -\vrho^\top \\
-\vrho & \vrho\vrho^\top
\end{bmatrix}
+
p_\dagger
\begin{bmatrix}
0 & 0 \\
0 & \diag(\vrho)-\vrho\vrho^\top
\end{bmatrix}.
\end{align}
Since
$
\nabla_{\vz}s=
\begin{bmatrix}
-1\\
\vrho
\end{bmatrix},
$
the first matrix is exactly
$\nabla_{\vz}s\,\nabla_{\vz}s^\top,$
which gives \eqref{eq:fgn_softmax_cov_decomposition}.
The residual is PSD since
$\diag(\vrho)-\vrho\vrho^\top$
is the covariance matrix of a categorical random variable with probabilities
$\vrho$.

For the parameter-space identity, multiply on the left and right by
$\mJ_z^\top$ and $\mJ_z$:
\begin{align}
\mJ_z^\top\bigl(\diag(\vp)-\vp\vp^\top\bigr)\mJ_z
&=
p_\star p_\dagger\,\mJ_z^\top \nabla_{\vz}s\,\nabla_{\vz}s^\top \mJ_z
+
p_\dagger\,
\mJ_{-\star}^\top
\bigl(\diag(\vrho)-\vrho\vrho^\top\bigr)
\mJ_{-\star}.
\end{align}
Using
$\mJ_s = (\nabla_{\vz}s)^\top \mJ_z$,
the first term equals
$p_\star p_\dagger\,\mJ_s^\top \mJ_s = \mH^{\mathrm{FGN}}$,
which yields \eqref{eq:fgn_ggn_decomposition}.

\subsection{Proof of Proposition~\ref{prop:fgn_decomposition_consequences}}
\label{apx:fgn_consequences}

We prove the items in order.

\paragraph{Item 1.}
Since $p_\star p_\dagger>0$ for finite logits and
$\mJ_s^\top\mJ_s\succeq 0$, we have
$\mH^{\mathrm{FGN}}\succeq 0$. 
The matrix
$\diag(\vrho)-\vrho\vrho^\top$
is the covariance matrix of a categorical distribution with probabilities
$\vrho$, and is therefore PSD. Hence the residual term in
\eqref{eq:fgn_ggn_decomposition} is PSD, which implies
\[
0 \preceq \mH^{\mathrm{FGN}} \preceq \mH^{\mathrm{GGN}}.
\]

\paragraph{Item 2.}
If $C=2$, then the competitor set has size one, so $\vrho=[1]$.
Therefore
$\diag(\vrho)-\vrho\vrho^\top = 0,$
and \eqref{eq:fgn_ggn_decomposition} reduces to
$\mH^{\mathrm{FGN}}=\mH^{\mathrm{GGN}}.$

\paragraph{Item 3.}
For any probability vector $\vrho$, the covariance matrix
$\diag(\vrho)-\vrho\vrho^\top$
vanishes if and only if the distribution is degenerate, that is,
if and only if $\vrho$ is one-hot.
Hence the residual in \eqref{eq:fgn_softmax_cov_decomposition} vanishes
if and only if $\vrho$ is one-hot, proving the logit-space exactness statement.

Sufficiency for parameter-space exactness is immediate from
\eqref{eq:fgn_ggn_decomposition}.
The converse need not hold in parameter space because multiplication by
$\mJ_{-\star}$ can annihilate a nonzero logit-space residual.
For example, if all competitor Jacobian rows are identical, then
$\mJ_{-\star}$
has rank at most one and
$
\mJ_{-\star}^\top
(\diag(\vrho)-\vrho\vrho^\top)
\mJ_{-\star}=0
$
for every $\vrho$, even when $\vrho$ is diffuse.

\paragraph{Item 4.}
Let $\mJ_1,\dots,\mJ_{C-1}$ denote the rows of $\mJ_{-\star}$, and define
$\bar{\mJ}_{\rho}\coloneqq \vrho^\top \mJ_{-\star}
=
\sum_{j=1}^{C-1}\rho_j \mJ_j.$
Then
\begin{align}
\mJ_{-\star}^\top\bigl(\diag(\vrho)-\vrho\vrho^\top\bigr)\mJ_{-\star}
&=
\sum_{j=1}^{C-1}\rho_j\,\mJ_j^\top \mJ_j
-
\bar{\mJ}_{\rho}^\top \bar{\mJ}_{\rho}.
\label{eq:apx_fgn_cov_expand}
\end{align}
On the other hand,
\begin{align}
\sum_{j=1}^{C-1}\rho_j\,
(\mJ_j-\bar{\mJ}_{\rho})^\top
(\mJ_j-\bar{\mJ}_{\rho})
&=
\sum_{j=1}^{C-1}\rho_j\,\mJ_j^\top\mJ_j
-
\bar{\mJ}_{\rho}^\top\bar{\mJ}_{\rho},
\end{align}
because $\sum_j \rho_j = 1$ and
$\sum_j \rho_j \mJ_j = \bar{\mJ}_{\rho}$.
Combining this identity with \eqref{eq:fgn_ggn_decomposition} proves
\eqref{eq:fgn_parameter_space_gap}.

\paragraph{Diagnostic $\xi$.}
The trace identity follows immediately:
\begin{align}
\Tr\bigl(\diag(\vrho)-\vrho\vrho^\top\bigr)
&=
\sum_{j=1}^{C-1}\rho_j - \sum_{j=1}^{C-1}\rho_j^2
=
1-\norm{\vrho}_2^2
=
\xi.
\end{align}
Because $\vrho$ is a probability vector on $C-1$ coordinates,
$\norm{\vrho}_2^2\in[1/(C-1),1]$, with the lower bound attained by the
uniform distribution and the upper bound attained by one-hot distributions.
Therefore
\begin{align}
0 \le \xi \le 1-\frac{1}{C-1}.
\end{align}
Finally, the dropped logit-space term in
\eqref{eq:fgn_softmax_cov_decomposition} has trace
\begin{align}
\Tr\left(
p_\dagger
\begin{bmatrix}
0 & 0 \\
0 & \diag(\vrho)-\vrho\vrho^\top
\end{bmatrix}
\right)
&=
p_\dagger\,\xi.
\end{align}
Hence, for fixed $p_\dagger$, the trace of the logit-space mismatch is
maximized by the most diffuse competitor distribution and vanishes when
the competitor distribution is one-hot.

\subsection{Batch gradient and curvature identities}
\label{apx:fgn_batch}

Let
\[
\Ls(\vw;\calB)=\frac{1}{b}\sum_{i=1}^b \ell(\vw;\vx_i,\vy_i).
\]
For each example $i$, Section~\ref{sec:margin} gives the exact gradient
\[
\nabla_{\vw}\ell(\vw;\vx_i,\vy_i)
=
p_\dagger^{(i)}\,\mJ_{s,i}^\top,
\qquad
\mJ_{s,i}(\vw)=\nabla_{\vw}s_i(\vw)^\top.
\]
Here the subscript $i$ indexes examples, not competitor classes.
Averaging over the batch yields
\begin{align}
\nabla_{\vw}\Ls(\vw;\calB)
&=
\frac{1}{b}\sum_{i=1}^b p_\dagger^{(i)}\,\mJ_{s,i}^\top
=
\frac{1}{b}\,\mJ(\vw)^\top \vr(\vw),
\end{align}
where
$r_i(\vw)=p_\dagger^{(i)}(\vw)$
and
$\mJ(\vw)$
stacks the rows $\mJ_{s,i}(\vw)$.

Likewise, the single-example FGN curvature is
\[
\mH_i^{\mathrm{FGN}}(\vw)
=
p_\star^{(i)}p_\dagger^{(i)}\,\mJ_{s,i}^\top \mJ_{s,i}.
\]
Averaging again gives
\begin{align}
\mH^{\mathrm{FGN}}(\vw;\calB)
&=
\frac{1}{b}\sum_{i=1}^b
p_\star^{(i)}p_\dagger^{(i)}\,\mJ_{s,i}^\top \mJ_{s,i}
=
\frac{1}{b}\,\mJ(\vw)^\top \mQ(\vw)\,\mJ(\vw),
\end{align}
where
$\mQ(\vw)=\diag(q_1(\vw),\dots,q_b(\vw))$
with
$q_i(\vw)=p_\star^{(i)}(\vw)p_\dagger^{(i)}(\vw).$
For finite logits, each $q_i(\vw)\in(0,1/4]$, so $\mQ(\vw)$ is diagonal and strictly positive.

%%%%%%%%%%%%%%%%%%%%%%%%%%%%%%%%%%%%%%%%%%%%%%%%%%%%%%%%%%%%%%%%%%%%%%%%%%%%%%%
%%%%%%%%%%%%%%%%%%%%%%%%%%%%%%%%%%%%%%%%%%%%%%%%%%%%%%%%%%%%%%%%%%%%%%%%%%%%%%%

\section{Derivations for the FGN Solver}
\label{apx:solver}

This appendix establishes the equivalence between the parameter-space and row-space systems,
gives the computational and storage accounting,
relates row-space residuals to parameter-space residuals,
and provides a simplified convergence analysis for the preconditioned update.
The equivalence and residual identities use the shared-batch quantities 
from Section~\ref{sec:fgn_row_system},
whereas the convergence analysis in
Appendices~\ref{apx:sec4_convergence} and~\ref{apx:sec4_inexact}
assumes independently sampled gradient and curvature mini-batches.

\subsection{Equivalence of parameter- and row-space solves}
\label{apx:sec4_equivalence}

\begin{lemma}
Assume $\lambda>0$.
Then the damped parameter-space system \eqref{eq:fgn_damped_system_param},
the unwhitened row-space system \eqref{eq:fgn_row_system_unwhitened},
and the whitened row-space system \eqref{eq:fgn_row_system_whitened}
are equivalent via the representations
\begin{align}
\vd = -\mJ^\top \vdelta = -\tilde{\mJ}^\top \vu.
\end{align}
In particular, solving \eqref{eq:fgn_row_system_whitened} and backprojecting with
\eqref{eq:fgn_direction_from_u} returns exactly the unique solution of
\eqref{eq:fgn_damped_system_param}.
\end{lemma}

\begin{proof}
Recall the definitions
\begin{align}
\vg &= \frac{1}{b}\,\mJ^\top \vr,
&
\mH^{\mathrm{FGN}}
&= \frac{1}{b}\,\mJ^\top \mQ \mJ,
\\
\tilde{\mJ} &\coloneqq \mQ^{1/2}\mJ,
&
\tilde{\vr} &\coloneqq \mQ^{-1/2}\vr,
&
\mK &\coloneqq \tilde{\mJ}\tilde{\mJ}^\top.
\end{align}
Let
\begin{align}
\mA \coloneqq \mH^{\mathrm{FGN}}+\lambda \mI_{d}
=
\frac{1}{b}\,\mJ^\top \mQ \mJ + \lambda \mI_{d}.
\end{align}
Because $\mQ \succ 0$, we have $\mJ^\top \mQ \mJ \succeq 0$, and since $\lambda>0$,
the parameter-space operator $\mA$ is symmetric positive definite.
Hence the parameter-space system has a unique solution.

The whitened row-space operator
\begin{align}
\mK + b\,\lambda \mI_b
\end{align}
is also symmetric positive definite, so the whitened row-space system has a unique solution.
Moreover,
\begin{align}
\mQ \mJ \mJ^\top + b\,\lambda \mI_b
&=
\mQ^{1/2}\bigl(\mK + b\,\lambda \mI_b\bigr)\mQ^{-1/2},
\label{eq:app_unwhitened_similarity}
\end{align}
so the unwhitened row-space operator is similar to the SPD matrix
$\mK + b\,\lambda \mI_b$
and is therefore invertible.

If $\vdelta$ solves the unwhitened row-space system
\begin{align}
\bigl(\mQ \mJ \mJ^\top + b\,\lambda \mI_b\bigr)\vdelta
&=
\vr
\end{align}
and we set
\begin{align}
\vd \coloneqq -\mJ^\top \vdelta,
\end{align}
then
\begin{align}
\bigl(\mH^{\mathrm{FGN}}+\lambda \mI_d\bigr)\vd
&=
\left(\frac{1}{b}\,\mJ^\top \mQ \mJ + \lambda \mI_d\right)(-\mJ^\top \vdelta)
\nonumber\\
&=
-\frac{1}{b}\,\mJ^\top
\bigl(\mQ \mJ \mJ^\top + b\,\lambda \mI_b\bigr)\vdelta
=
-\frac{1}{b}\,\mJ^\top \vr
=
-\vg.
\end{align}
So the backprojected unwhitened row-space solution solves the parameter-space system.

If $\vu$ solves the whitened row-space system
\begin{align}
\bigl(\mK + b\,\lambda \mI_b\bigr)\vu
&=
\tilde{\vr}
\end{align}
and we set
\begin{align}
\vd \coloneqq -\tilde{\mJ}^\top \vu,
\end{align}
then
\begin{align}
\bigl(\mH^{\mathrm{FGN}}+\lambda \mI_d\bigr)\vd
&=
\left(\frac{1}{b}\,\tilde{\mJ}^\top \tilde{\mJ} + \lambda \mI_d\right)(-\tilde{\mJ}^\top \vu)
\nonumber\\
&=
-\frac{1}{b}\,\tilde{\mJ}^\top
\bigl(\tilde{\mJ}\tilde{\mJ}^\top + b\,\lambda \mI_b\bigr)\vu
=
-\frac{1}{b}\,\tilde{\mJ}^\top \tilde{\vr}
=
-\vg.
\end{align}
So the backprojected whitened row-space solution also solves the same parameter-space system.

Conversely, let $\vd$ be the unique solution of the parameter-space system.
Let $\vdelta$ be the unique solution of the unwhitened row-space system.
By the implication above,
$\vd' \coloneqq -\mJ^\top \vdelta$
also solves the parameter-space system.
Uniqueness therefore implies
$\vd' = \vd.$
Finally, the linear change of variables
\begin{align}
\vu = \mQ^{-1/2}\vdelta,
\qquad
\vdelta = \mQ^{1/2}\vu,
\end{align}
is invertible, so the unwhitened and whitened row-space systems are equivalent
and yield the same backprojected direction.
\end{proof}

\subsection{Cost accounting}
\label{apx:solver_cost}

Recall the whitened row-space operator
\begin{align}
\mB\vv
=
\tilde{\mJ}\bigl(\tilde{\mJ}^{\top}\vv\bigr)
+
b\lambda\vv,
\qquad
\tilde{\mJ}=\mQ^{1/2}\mJ,
\qquad
\vv\in\R^b .
\end{align}
A matrix-free application proceeds as
\begin{align}
\va &= \mQ^{1/2}\vv,
&
\vbeta &= \mJ^\top\va,
&
\vgamma &= \mJ\vbeta,
&
\mB\vv &= \mQ^{1/2}\vgamma+b\lambda\vv .
\end{align}
Thus one application of $\mB$ requires one vector-Jacobian product 
and one
Jacobian-vector product with the batched true-vs-rest margin map, 
plus $\mathcal{O}(b)$ vector operations for diagonal scalings and damping.

In autodiff implementations, 
the batched margin map can be linearized once at
the current parameters, and the resulting JVP/VJP closures can be reused
across CG iterations.

For standard row-space CG with zero initialization and at most
$N_{\mathrm{CG}}$ iterations, 
the additional solve work is bounded by
\begin{align}
\mathrm{cost}(\text{extra FGN solve})
&\le
\mathrm{cost}(\mathrm{linearize})
+
N_{\mathrm{CG}}
\Bigl[
\mathrm{cost}(\mathrm{JVP}_s)
+
\mathrm{cost}(\mathrm{VJP}_s)
\Bigr]
\nonumber\\
&\qquad
+
\mathrm{cost}(\mathrm{VJP}_s)
+
\mathcal{O}\bigl((N_{\mathrm{CG}}+1)b\bigr),
\label{eq:apx_fgn_cost_extra_solve}
\end{align}
where the subscript $s$ denotes the batched scalar-margin map,
and the final VJP is the backprojection
$\vd=-\tilde{\mJ}^{\top}\vu$.
The linearization term includes the one-time setup of reusable derivative routines.
For example, \texttt{jax.linearize} in JAX~\citep{jax2018github} evaluates 
the margin map at the current parameters and returns a reusable JVP function.
The costs of linearization, JVPs, and VJPs depend on the model and implementation.

The row-space system has dimension $b$, 
and no $d\times d$ matrix is formed.
A dense direct row-space solve would require explicitly forming
$\mK+b\lambda\mI_b$ and paying $\mathcal{O}(b^3)$ for a factorization. 
In the generic matrix-free setting, 
explicitly forming $\mK$ itself would require
probing the row-space operator with many right-hand sides. 
Matrix-free CG
instead uses only repeated applications of~\eqref{eq:fgn_row_operator_whitened}.

Table~\ref{tab:solver_cost_comparison}
summarizes derivative setup, 
gradient and solve operations,
and persistent optimizer vectors.
The final $\mathcal{O}(d)$ parameter update is common to all methods.
The FGN row describes standard matrix-free row-space CG.
For full GGN, 
we consider an SGN implementation~\citep{gargiani2020promise} 
using parameter-space CG.
The CG counts assume zero initialization and no additional preconditioning.

\begin{table*}[t]
\centering
\caption{Computational and storage comparison of FGN, full GGN,
and first-order methods.}
\label{tab:solver_cost_comparison}
\small
\setlength{\tabcolsep}{5pt}
\begin{tabular}{
    @{}
    p{0.14\textwidth}
    p{0.09\textwidth}
    p{0.50\textwidth}
    p{0.18\textwidth}
    @{}
}
\toprule
Method & Solve size & Core operations & Persistent vectors \\
\midrule

FGN
& $b$
& Margin linearization;
at most $N_{\mathrm{CG}}$ margin JVP/VJP pairs;
$\mathcal{O}(N_{\mathrm{CG}}b)$ row-space vector work;
one final VJP
& $\mathcal{O}(b)$ CG vectors
\\

Full GGN
& $d$
& Logit derivative setup and gradient evaluation;
at most $N_{\mathrm{CG}}$ logit JVP/VJP pairs;
$\mathcal{O}(N_{\mathrm{CG}}bC)$ covariance work;
$\mathcal{O}(N_{\mathrm{CG}}d)$ CG vector work
& $\mathcal{O}(d)$ CG vectors
\\

SGD
& None
& Gradient evaluation
& None; $\mathcal{O}(d)$ with momentum
\\

Adam
& None
& Gradient evaluation and $\mathcal{O}(d)$ moment updates
& Two $\mathcal{O}(d)$ moment vectors
\\

\bottomrule
\end{tabular}
\end{table*}

FGN forms $\tilde{\vr}=\mQ^{-1/2}\vr$ directly from the probabilities,
so it does not require a separate parameter-space gradient evaluation
before CG.
For full GGN, each logit JVP produces a $b\times C$ perturbation.
For each example, multiplication by the softmax covariance is
\[
\vv
\longmapsto
\vp\odot\vv-\vp(\vp^\top\vv),
\qquad
\vv\in\R^C,
\]
where $\odot$ denotes elementwise multiplication.
These multiplications cost $\mathcal{O}(C)$ per example,
or $\mathcal{O}(bC)$ per mini-batch.
A logit VJP then maps the result back to parameter space.
No dense $C\times C$ covariance matrix is formed.

The persistent vector counts include only CG vectors 
or optimizer moment vectors; 
they exclude model parameters, gradients,
autodiff temporaries, and stored linearizations.
They therefore do not represent peak training memory.
FGN still uses parameter-sized intermediate vectors,
and its linearization can retain model activations.
Both second-order methods also evaluate all $C$ logits and probabilities.
CG convergence depends on the spectrum and right-hand side,
so the smaller FGN system does not by itself imply fewer iterations.

\paragraph{Lifted row-space CG.}
Let $\boldsymbol{p}_k$ be the CG search direction and $a_k$
its step length.
Each operator application already computes
$\boldsymbol{q}_k=\tilde{\mJ}^\top\boldsymbol{p}_k$.
Starting from $\vd_0=\vzero$, 
the lifted variant accumulates
\[
\vd_{k+1}=\vd_k-a_k\boldsymbol{q}_k.
\]
Since $\vu_{k+1}=\vu_k+a_k\boldsymbol{p}_k$,
we have $\vd_k=-\tilde{\mJ}^\top\vu_k$ in exact arithmetic.
Thus lifting eliminates the final VJP 
at the cost of storing
a $d$-dimensional accumulator updated at each iteration.

\subsection{A simplified stochastic preconditioned viewpoint}
\label{apx:sec4_convergence}

We record one standard convergence template for the solver interpretation from Section~\ref{sec:fgn_convergence}.
It applies only to a simplified variant with fixed damping and independent curvature and gradient mini-batches.

Let $\train$ be a finite training set and define the empirical risk
\begin{align}
\Risk(\vw)
\coloneqq
\frac{1}{|\train|}
\sum_{(\vx,\vy)\in\train}
\ell(\vw;\vx,\vy).
\end{align}
For an i.i.d. mini-batch $\calB$ of size $b$, the batch loss
$\Ls(\vw;\calB)$
satisfies
\begin{align}
\E_{\calB}\big[\Ls(\vw;\calB)\big]
=
\Risk(\vw).
\end{align}

At iteration $t$, 
let $\hat{\calB}_t$ be an i.i.d. curvature mini-batch and $\calB_t$ 
an independent i.i.d. gradient mini-batch.
We form
\begin{align}
\mH_t^{\mathrm{FGN}}
&\coloneqq
\mH^{\mathrm{FGN}}(\vw_t;\hat{\calB}_t),
&
\vg_t
&\coloneqq
\nabla_{\vw}\Ls(\vw_t;\calB_t),
\end{align}
fix a damping parameter $\lambda>0$, define
\begin{align}
\mP_t
\coloneqq
\bigl(\mH_t^{\mathrm{FGN}}+\lambda \mI_d\bigr)^{-1},
\end{align}
and update
\begin{align}
\vw_{t+1}
=
\vw_t - \lr_t \mP_t \vg_t.
\label{eq:app_solver_exact_update}
\end{align}
Let $\cF_t$ denote the sigma-algebra generated by
$\vw_0,\hat{\calB}_0,\calB_0,\ldots,\hat{\calB}_{t-1},\calB_{t-1},\hat{\calB}_t.$
Then $\mP_t$ is $\cF_t$-measurable and $\calB_t$ is independent of $\cF_t$.

\begin{assumption}[Smooth risk]
\label{ass:app_solver_smooth}
The risk $\Risk$ is bounded below and has $L_g$-Lipschitz gradient:
for all $\vw,\vu$,
\begin{align}
\|\nabla \Risk(\vw)-\nabla \Risk(\vu)\|
\le
L_g \|\vw-\vu\|.
\end{align}
\end{assumption}

\begin{assumption}[Unbiased stochastic gradients]
\label{ass:app_solver_unbiased}
For all $t$,
\begin{align}
\E[\vg_t \mid \cF_t]
&=
\nabla \Risk(\vw_t),
&
\E[\|\vg_t\|^2 \mid \cF_t]
&\le
G^2
\end{align}
almost surely, for some constant $G<\infty$.
\end{assumption}

\begin{assumption}[Uniform spectral bounds]
\label{ass:app_solver_spectrum}
There exist constants
$0<\underline{\kappa}\le \overline{\kappa}<\infty$
such that, almost surely for all $t$,
\begin{align}
\underline{\kappa}\,\mI_d
\preceq
\mP_t
\preceq
\overline{\kappa}\,\mI_d.
\end{align}
\end{assumption}

\begin{assumption}[Step sizes]
\label{ass:app_solver_stepsizes}
The step sizes are positive and satisfy
\begin{align}
\sum_{t=0}^{\infty} \lr_t
=
\infty,
\qquad
\sum_{t=0}^{\infty} \lr_t^2
<
\infty.
\end{align}
\end{assumption}

\begin{theorem}[A simplified stochastic FGN guarantee]
\label{thm:app_solver_simplified}
Suppose Assumptions~\ref{ass:app_solver_smooth}--\ref{ass:app_solver_stepsizes} hold, 
and let $\{\vw_t\}$ be generated by \eqref{eq:app_solver_exact_update}.
Then:
\begin{enumerate}
    \item the sequence $\{\Risk(\vw_t)\}$ converges almost surely to a finite random variable;
    \item
    \begin{align}
    \sum_{t=0}^{\infty}
    \lr_t\,\E\big[\|\nabla \Risk(\vw_t)\|^2\big]
    <
    \infty;
    \end{align}
    \item in particular,
    \begin{align}
    \liminf_{t\to\infty}
    \E\big[\|\nabla \Risk(\vw_t)\|^2\big]
    =
    0.
    \end{align}
\end{enumerate}
\end{theorem}

\begin{proof}
By $L_g$-smoothness of $\Risk$,
\begin{align}
\Risk(\vw_{t+1})
&\le
\Risk(\vw_t)
+
\nabla \Risk(\vw_t)^\top(\vw_{t+1}-\vw_t)
+
\frac{L_g}{2}\|\vw_{t+1}-\vw_t\|^2.
\end{align}
Substituting
$\vw_{t+1}-\vw_t = -\lr_t \mP_t \vg_t$
gives
\begin{align}
\Risk(\vw_{t+1})
&\le
\Risk(\vw_t)
-
\lr_t \nabla \Risk(\vw_t)^\top \mP_t \vg_t
+
\frac{L_g \lr_t^2}{2}\|\mP_t \vg_t\|^2.
\end{align}
Taking conditional expectation given $\cF_t$ and using
Assumption~\ref{ass:app_solver_unbiased} yields
\begin{align}
\E\big[\Risk(\vw_{t+1}) \mid \cF_t\big]
&\le
\Risk(\vw_t)
-
\lr_t \nabla \Risk(\vw_t)^\top \mP_t \nabla \Risk(\vw_t)
+
\frac{L_g \lr_t^2}{2}
\E\big[\|\mP_t \vg_t\|^2 \mid \cF_t\big].
\end{align}
By Assumption~\ref{ass:app_solver_spectrum},
\begin{align}
\nabla \Risk(\vw_t)^\top \mP_t \nabla \Risk(\vw_t)
&\ge
\underline{\kappa}\,\|\nabla \Risk(\vw_t)\|^2,
\\
\|\mP_t \vg_t\|^2
&\le
\overline{\kappa}^2 \|\vg_t\|^2.
\end{align}
Using Assumption~\ref{ass:app_solver_unbiased} again,
\begin{align}
\E\big[\Risk(\vw_{t+1}) \mid \cF_t\big]
&\le
\Risk(\vw_t)
-
\underline{\kappa}\,\lr_t \|\nabla \Risk(\vw_t)\|^2
+
\frac{L_g \overline{\kappa}^2 G^2}{2}\,\lr_t^2.
\label{eq:app_solver_supermartingale}
\end{align}
This is an almost-supermartingale recursion with summable error term because
$\sum_t \lr_t^2 < \infty.$
Applying the Robbins-Siegmund theorem~\citep{robbins1971convergence}
to the nonnegative sequence
$\Risk(\vw_t)-\inf_{\vw}\Risk(\vw)$,
we obtain that $\{\Risk(\vw_t)\}$ converges almost surely
to a finite limit and
\begin{align}
\sum_{t=0}^{\infty}
\lr_t\,\|\nabla \Risk(\vw_t)\|^2
<
\infty
\qquad
\text{almost surely.}
\end{align}
Taking full expectations in \eqref{eq:app_solver_supermartingale} and summing over $t$ also gives
\begin{align}
\sum_{t=0}^{\infty}
\lr_t\,\E\big[\|\nabla \Risk(\vw_t)\|^2\big]
<
\infty.
\end{align}
Since
$\sum_t \lr_t = \infty$
and each term is nonnegative, this implies
\begin{align}
\liminf_{t\to\infty}
\E\big[\|\nabla \Risk(\vw_t)\|^2\big]
=
0.
\end{align}
\end{proof}

\subsection{Residual transfer from row space to parameter space}
\label{apx:sec4_residual_transfer}

For $\lambda_t>0$, 
let $\vu_t \in \R^b$ be any row-space iterate and define
\begin{align}
\ve_t
&\coloneqq
\bigl(\mK_t + b\,\lambda_t \mI_b\bigr)\vu_t - \tilde{\vr}_t,
&
\vd_t &\coloneqq -\tilde{\mJ}_t^\top \vu_t.
\end{align}
Then
\begin{align}
\bigl(\mH_t^{\mathrm{FGN}}+\lambda_t\mI_d \bigr)\vd_t + \vg_t
&=
\left(\frac{1}{b}\,\tilde{\mJ}_t^\top \tilde{\mJ}_t + \lambda_t \mI_d \right)(-\tilde{\mJ}_t^\top \vu_t)
+ \frac{1}{b}\,\tilde{\mJ}_t^\top \tilde{\vr}_t
\nonumber\\
&=
-\frac{1}{b}\,\tilde{\mJ}_t^\top
\bigl(\tilde{\mJ}_t\tilde{\mJ}_t^\top + b\,\lambda_t \mI_b\bigr)\vu_t
+ \frac{1}{b}\,\tilde{\mJ}_t^\top \tilde{\vr}_t
=
-\frac{1}{b}\,\tilde{\mJ}_t^\top \ve_t.
\end{align}
It is therefore natural to define the parameter-space residual by
\begin{align}
\boldsymbol{\varepsilon}_t
&\coloneqq
\bigl(\mH_t^{\mathrm{FGN}}+\lambda_t\mI_d \bigr)\vd_t + \vg_t
=
-\frac{1}{b}\,\tilde{\mJ}_t^\top \ve_t.
\label{eq:app_parameter_residual_from_row_residual}
\end{align}

Taking Euclidean norms and using the operator norm gives
\begin{align}
\|\boldsymbol{\varepsilon}_t\|
&=
\frac{1}{b}\,\|\tilde{\mJ}_t^\top \ve_t\|
\le
\frac{1}{b}\,\|\tilde{\mJ}_t^\top\|\,\|\ve_t\|
=
\frac{1}{b}\,\|\tilde{\mJ}_t\|\,\|\ve_t\|.
\label{eq:app_parameter_residual_norm_bound}
\end{align}
Hence a small row-space residual implies a small parameter-space residual,
up to the multiplicative factor $\|\tilde{\mJ}_t\|/b$.

This bound is the basic link between the row-space residual monitored by CG
and parameter-space accuracy.
However, a relative row-space stopping rule of the form
$\|\ve_t\| \le \tau_t \|\tilde{\vr}_t\|$
does not by itself imply the relative parameter-space residual condition
used in Lemma~\ref{lem:app_solver_inexact} without additional control of
$\|\tilde{\mJ}_t\|$ and $\|\vg_t\|$.
For that reason, the inexact solve lemma
states the parameter-space residual condition explicitly.

\subsection{Inexact linear solves}
\label{apx:sec4_inexact}

We return to the independent-batch update in
Theorem~\ref{thm:app_solver_simplified}
and allow its parameter-space system to be solved approximately.
The following lemma gives sufficient conditions
in terms of the parameter-space residual.
\begin{lemma}[Inexact solves under a relative residual condition]
\label{lem:app_solver_inexact}
Assume Assumptions~\ref{ass:app_solver_smooth}--\ref{ass:app_solver_stepsizes}
hold, with the same filtration and mini-batch independence as in
Theorem~\ref{thm:app_solver_simplified}. 
Let
\begin{align}
\mA_t
\coloneqq
\mH_t^{\mathrm{FGN}}+\lambda \mI_d,
\qquad
\mP_t
\coloneqq
\mA_t^{-1},
\end{align}
and consider updates
\begin{align}
\vw_{t+1}
=
\vw_t + \lr_t \vd_t,
\end{align}
where $\vd_t$ is produced by an iterative solver applied to
$\mA_t \vd = -\vg_t.$
Define the parameter-space residual
\begin{align}
\boldsymbol{\varepsilon}_t
\coloneqq
\mA_t \vd_t + \vg_t.
\end{align}
Assume that for some deterministic nonnegative sequence $\{\eta_t\}$,
\begin{align}
\|\boldsymbol{\varepsilon}_t\|
\le
\eta_t \|\vg_t\|
\qquad
\text{almost surely,}
\label{eq:app_solver_relative_residual}
\end{align}
with
\begin{align}
\sup_t \eta_t < \infty,
\qquad
\sum_{t=0}^{\infty} \lr_t \eta_t^2 < \infty.
\end{align}
Then the conclusions of Theorem~\ref{thm:app_solver_simplified} remain valid.
\end{lemma}

\begin{proof}
Since
$\boldsymbol{\varepsilon}_t = \mA_t \vd_t + \vg_t$
and
$\mP_t = \mA_t^{-1},$
the inexact direction can be written as
\begin{align}
\vd_t
=
-\mP_t \vg_t + \mP_t \boldsymbol{\varepsilon}_t.
\label{eq:app_solver_inexact_direction}
\end{align}
Using smoothness of $\Risk$,
\begin{align}
\Risk(\vw_{t+1})
&\le
\Risk(\vw_t)
+
\lr_t \nabla \Risk(\vw_t)^\top \vd_t
+
\frac{L_g \lr_t^2}{2}\|\vd_t\|^2.
\end{align}
Substituting \eqref{eq:app_solver_inexact_direction} and taking conditional
expectation given $\cF_t$ gives
\begin{align}
\E\big[\Risk(\vw_{t+1}) \mid \cF_t\big]
&\le
\Risk(\vw_t)
-
\lr_t \nabla \Risk(\vw_t)^\top \mP_t \nabla \Risk(\vw_t)
\nonumber\\
&\qquad
+
\lr_t
\E\big[
\nabla \Risk(\vw_t)^\top \mP_t \boldsymbol{\varepsilon}_t
\mid \cF_t
\big]
+
\frac{L_g \lr_t^2}{2}
\E\big[\|\vd_t\|^2 \mid \cF_t\big],
\end{align}
where we used
$\E[\vg_t\mid \cF_t]=\nabla \Risk(\vw_t)$
and the fact that $\mP_t$ is $\cF_t$-measurable.

For each realization, Cauchy-Schwarz in the $\mP_t$ metric and Young's
inequality give
\begin{align}
\nabla \Risk(\vw_t)^\top \mP_t \boldsymbol{\varepsilon}_t
&\le
\frac{1}{2}\nabla \Risk(\vw_t)^\top \mP_t \nabla \Risk(\vw_t)
+
\frac{1}{2}\boldsymbol{\varepsilon}_t^\top \mP_t
\boldsymbol{\varepsilon}_t
\nonumber\\
&\le
\frac{1}{2}\nabla \Risk(\vw_t)^\top \mP_t \nabla \Risk(\vw_t)
+
\frac{\overline{\kappa}}{2}\eta_t^2 \|\vg_t\|^2 .
\end{align}
Therefore
\begin{align}
\E\big[
\nabla \Risk(\vw_t)^\top \mP_t \boldsymbol{\varepsilon}_t
\mid \cF_t
\big]
&\le
\frac{1}{2}\nabla \Risk(\vw_t)^\top \mP_t \nabla \Risk(\vw_t)
+
\frac{\overline{\kappa}}{2}\eta_t^2
\E[\|\vg_t\|^2\mid \cF_t].
\end{align}
Similarly, using
$\vd_t=-\mP_t\vg_t+\mP_t\boldsymbol{\varepsilon}_t$,
\begin{align}
\|\vd_t\|^2
\le
2\|\mP_t\vg_t\|^2
+
2\|\mP_t\boldsymbol{\varepsilon}_t\|^2
\le
2\overline{\kappa}^2(1+\eta_t^2)\|\vg_t\|^2,
\end{align}
and hence
\begin{align}
\E[\|\vd_t\|^2\mid \cF_t]
&\le
2\overline{\kappa}^2(1+\eta_t^2)
\E[\|\vg_t\|^2\mid \cF_t].
\end{align}
Substituting the last two bounds gives
\begin{align}
\E\big[\Risk(\vw_{t+1}) \mid \cF_t\big]
&\le
\Risk(\vw_t)
-
\frac{1}{2}\lr_t \nabla \Risk(\vw_t)^\top \mP_t \nabla \Risk(\vw_t)
+
\frac{\overline{\kappa}}{2}\lr_t \eta_t^2
\E[\|\vg_t\|^2 \mid \cF_t]
\nonumber\\
&\qquad
+
L_g \overline{\kappa}^2 \lr_t^2 (1+\eta_t^2)
\E[\|\vg_t\|^2 \mid \cF_t].
\end{align}
Using Assumptions~\ref{ass:app_solver_unbiased} and
\ref{ass:app_solver_spectrum},
\begin{align}
\E\big[\Risk(\vw_{t+1}) \mid \cF_t\big]
&\le
\Risk(\vw_t)
-
\frac{\underline{\kappa}}{2}\lr_t \|\nabla \Risk(\vw_t)\|^2
+
\frac{\overline{\kappa}G^2}{2}\lr_t \eta_t^2
+
L_g \overline{\kappa}^2 G^2 \lr_t^2 (1+\eta_t^2).
\label{eq:app_solver_inexact_supermartingale}
\end{align}
Because
$\sup_t \eta_t < \infty$,
$\sum_t \lr_t^2 < \infty$,
and
$\sum_t \lr_t \eta_t^2 < \infty$,
the error terms in
\eqref{eq:app_solver_inexact_supermartingale}
are summable.
The same Robbins-Siegmund argument as in the exact-solve case yields almost
sure convergence of $\{\Risk(\vw_t)\}$ and
\[
\sum_{t=0}^\infty
\lr_t \|\nabla \Risk(\vw_t)\|^2
<\infty
\qquad
\text{almost surely.}
\]
Taking full expectations in
\eqref{eq:app_solver_inexact_supermartingale}
and summing over $t$ also gives
\[
\sum_{t=0}^{\infty}
\lr_t\,\E\big[\|\nabla \Risk(\vw_t)\|^2\big]
<
\infty.
\]
Since $\sum_t \lr_t=\infty$ and the terms are nonnegative,
\[
\liminf_{t\to\infty}
\E\big[\|\nabla \Risk(\vw_t)\|^2\big]
=
0.
\]
\end{proof}

% %%%%%%%%%%%%%%%%%%%%%%%%%%%%%%%%%%%%%%%%%%%%%%%%%%%%%%%%%%%%%%%%%%%%%%%%%%%%%%%
% %%%%%%%%%%%%%%%%%%%%%%%%%%%%%%%%%%%%%%%%%%%%%%%%%%%%%%%%%%%%%%%%%%%%%%%%%%%%%%%

\section{Algorithms}
\label{apx:fgn_algorithm}

% \subsection{FGN Algorithm}

Algorithm~\ref{alg:fgn_core} summarizes the core damped FGN step 
analyzed in Section~\ref{sec:solver}.

\begin{algorithm}[t]
    \caption{Core Fast Gauss-Newton (FGN) step}
    \label{alg:fgn_core}
    \begin{algorithmic}[1]
        \Require mini-batch $\calB_t$, parameters $\vw_t$, damping
        $\lambda_t>0$, step size $\lr_t$, 
        relative CG tolerance $\eps$, 
        maximum CG iterations $N_{\mathrm{CG}}$.
        
        \Ensure updated parameters $\vw_{t+1}$.

        \State Compute logits, softmax probabilities, margins
        $s_i=z_\dagger^{(i)}-z_\star^{(i)}$, and weights
        $q_i=p_\star^{(i)}p_\dagger^{(i)}$ for $i=1,\ldots,b$.

        \State Define
        \[
            \mQ_t \coloneqq \diag(q_1,\ldots,q_b),
            \qquad
            (\tilde{\vr}_t)_i
            \coloneqq
            \sqrt{\frac{p_\dagger^{(i)}}{p_\star^{(i)}}}.
        \]

        \State Linearize the margin map at $\vw_t$ and construct
        its matrix-free JVP/VJP routines.
        Define the whitened Jacobian operator
        \[
            \tilde{\mJ}_t\coloneqq\mQ_t^{1/2}\mJ_t
        \]
        with $\mQ_t$ held fixed during JVP/VJP evaluations.

        \State Define the row-space operator
        \[
            \mB_t\vv
            \coloneqq
            \tilde{\mJ}_t\tilde{\mJ}_t^\top \vv
            +
            b\lambda_t\vv,
            \qquad
            \vv\in\R^b .
        \]

        \State Starting from $\vu=\vzero$, apply CG to
        \[
            \mB_t\vu=\tilde{\vr}_t,
        \]
        stopping when
        $\|\mB_t\vu-\tilde{\vr}_t\| \le \eps\|\tilde{\vr}_t\|$
        or after $N_{\mathrm{CG}}$ iterations,
        and set the final iterate to $\vu_t$.

        \State Recover the parameter-space direction
        \[
            \vd_t \coloneqq -\tilde{\mJ}_t^\top\vu_t .
        \]

        \State Update
        \[
            \vw_{t+1}\coloneqq \vw_t+\lr_t\vd_t .
        \]
    \end{algorithmic}
\end{algorithm}
Momentum, adaptive damping, preconditioning, CG warm starts, line search, 
and curvature reuse are implementation choices 
rather than part of the core
equivalence established in Section~\ref{sec:fgn_row_system}.

%%%%%%%%%%%%%%%%%%%%%%%%%%%%%%%%%%%%%%%%%%%%%%%%%%%%%%%%%%%%%%%%%%%%%%%%%%%%%%%
%%%%%%%%%%%%%%%%%%%%%%%%%%%%%%%%%%%%%%%%%%%%%%%%%%%%%%%%%%%%%%%%%%%%%%%%%%%%%%%

\section{Experiment Details}
\label{apx:experiment_details}

This appendix provides the experimental setups, hyperparameters,
implementation details, 
and reporting protocols for the three
experiments in Section~\ref{sec:exp}.
It also reports supporting diagnostics on competitor dispersion,
optimizer steps, and CG budget sensitivity.
All experiments were run on a single NVIDIA RTX A4000 GPU
with 16 GB of memory.

\subsection{Mechanism checks}
\label{apx:exp_mechanisms}

The three panels examine FGN from different angles: 
update cost, 
retained and discarded curvature, 
and the effect of damping on the resulting steps.

\paragraph{Panel (a): optimizer update timing.}
We use an affine softmax head on fixed Gaussian features.
For each example,
\[
    \vz_i=\mW^\top\vx_i+\vb,
    \qquad
    \vx_i\sim\mathcal{N}(\vzero,\mI_{2048}),
    \qquad
    \mW\in\R^{2048\times C},
    \qquad
    \vb\in\R^C.
\]
Equivalently, for a batch,
\[
    \mZ=\mX\mW+\mathbf{1}\vb^\top,
    \qquad
    \mX\in\R^{b\times2048}.
\]
We sweep
\[
    C\in\{2,4,8,16,32,64,128,256,512,1024,2048,4096\},
\]
using batch size $512$ and $32768$ fixed synthetic examples
with independently sampled uniform class labels.
FGN uses row-space CG, whereas SGN uses parameter-space CG.
Both second-order methods use fixed damping $\lambda=1$,
at most five CG iterations, tolerance $10^{-5}$, and no warm starts.

For each method and class count, 
we perform $128$ warmup updates,
followed by $640$ timed updates in ten synchronized blocks of $64$.
We report the mean and $\pm1$ standard deviation over five seeds, 
using the median per-update block time from each run.
The timings cover complete optimizer updates, 
including all logits,
and exclude data generation, loading, compilation, and logging.

\paragraph{Panel (b): output-space trace decomposition.}
We fix $p_\star=0.6$, set $p_\dagger=1-p_\star$, 
use $C=10$,
and vary the competitor dispersion
\[
    \xi=1-\|\vrho\|_2^2,
    \qquad
    0\le\xi\le1-\frac{1}{C-1}.
\]
The retained, dropped, and full output-space traces are
\[
    \tau_{\mathrm{ret}}=p_\star p_\dagger(2-\xi),
    \qquad
    \tau_{\mathrm{drop}}=p_\dagger\xi,
    \qquad
    \tau_{\mathrm{full}}=2p_\star p_\dagger+p_\dagger^2\xi.
\]
The shaded gap in Figure~\ref{fig:triptych}(b) is
$\tau_{\mathrm{full}}-\tau_{\mathrm{ret}}=\tau_{\mathrm{drop}}$.

\paragraph{Panel (c): agreement between damped steps.}
To isolate the effect of the dropped within-competitor covariance, 
we construct a two-dimensional example 
in which the gradient and FGN curvature remain fixed 
as competitor dispersion varies.

We use $C=10$ and $p_\star=0.70$.
Let $\mA\in\R^{C\times2}$ be the logit Jacobian,
with row $\va_c^\top$ for class $c$.
We set
\[
    \va_\star=(-1,\eta)^\top,
    \qquad
    \va_j=(1,\beta t_j)^\top,
    \qquad
    \eta=1.25,
    \qquad
    \beta=3.0,
\]
where $j=1,\ldots,9$ indexes the competing classes and
\[
    t_1=0,
    \qquad
    \{t_2,\ldots,t_9\}
    =
    \left\{
    \pm\tfrac14,\pm\tfrac12,\pm\tfrac34,\pm1
    \right\}.
\]
We vary the competitor distribution using
\[
    \vrho(\alpha)
    =
    (1-\alpha)\mathbf{e}_1
    +
    \alpha\frac{\mathbf{1}_{C-1}}{C-1},
\]
where $\mathbf{e}_1\in\R^{C-1}$ is the first standard basis vector.
We use $201$ equally spaced values of $\alpha$ from $10^{-5}$ to $1$,
and plot against $\xi(\alpha)=1-\|\vrho(\alpha)\|_2^2$.

In the coordinate order $(\star,-\star)$, 
define
\[
    \vp=
    \begin{bmatrix}
        p_\star\\
        p_\dagger\vrho
    \end{bmatrix},
    \qquad
    p_\dagger=1-p_\star,
    \qquad
    \ve_\star=(1,0,\ldots,0)^\top,
    \qquad
    \vg=\mA^\top(\vp-\ve_\star).
\]
The two curvature matrices are
\[
    \mH^{\mathrm{GGN}}
    =
    \mA^\top\bigl(\diag(\vp)-\vp\vp^\top\bigr)\mA
\]
and
\[
    \mH^{\mathrm{FGN}}
    =
    p_\star p_\dagger
    (\bar{\va}_\rho-\va_\star)(\bar{\va}_\rho-\va_\star)^\top,
    \qquad
    \bar{\va}_\rho=\mA_{-\star}^\top\vrho,
\]
where $\mA_{-\star}$ contains the competitor rows.
Since $t_1=0$ and $\sum_j t_j=0$, the mean competitor row
$\bar{\va}_\rho=(1,0)^\top$ 
is constant along the sweep.

For damping scales $s_\lambda\in\{0.01,0.1,1.0\}$,
we set
\[
    \lambda_s
    =
    s_\lambda\frac{\Tr(\mH^{\mathrm{FGN}}(\xi=0))}{2},
\]
using the limiting value at $\alpha\downarrow0$,
and hold $\lambda_s$ fixed along each curve.
We compute both damped directions by direct $2\times2$ solves:
\[
    \vd_{\mathrm{GGN}}
    =
    -(\mH^{\mathrm{GGN}}+\lambda_s\mI_2)^{-1}\vg,
    \qquad
    \vd_{\mathrm{FGN}}
    =
    -(\mH^{\mathrm{FGN}}+\lambda_s\mI_2)^{-1}\vg.
\]
Both directions are evaluated under the same full GGN damped
quadratic model, 
with predicted decrease
\[
    \Delta^{\mathrm{GGN}}_{\lambda_s}(\vd)
    \coloneqq
    -\left[
        \vg^\top\vd
        +
        \frac12\vd^\top\mH^{\mathrm{GGN}}\vd
        +
        \frac{\lambda_s}{2}\|\vd\|^2
    \right].
\]
Panel~(c) plots
\[
    R_{\lambda_s}(\xi)
    =
    \frac{\Delta^{\mathrm{GGN}}_{\lambda_s}(\vd_{\mathrm{FGN}})}
         {\Delta^{\mathrm{GGN}}_{\lambda_s}(\vd_{\mathrm{GGN}})}.
\]
A ratio of one means that the two damped directions agree.

\subsection{Frozen-feature multiclass head optimization}
\label{apx:fig2_headopt_details}

\paragraph{Data and features.}
We use the canonical Cars196~\citep{KrauseStarkDengFei-Fei_3DRR2013} 
train/test split. 
Each image is passed once through a frozen ImageNet-pretrained
ResNet-50~\citep{he2016deep} encoder, 
using the corresponding
TorchVision weight transform and no data augmentation.
The cached feature is the average-pooled pre-classifier representation
$\vh_i\in\R^{2048}$.
Features are standardized coordinatewise using the training-set
mean and standard deviation.
The cache contains $8144$ training features and $8041$ test features.

\paragraph{Head objective.}
Given feature $\vh_i\in\R^{2048}$ and label
$y_i\in\{1,\ldots,196\}$, 
the optimized head is
\[
    \vz_i=\mW^\top\vh_i+\vb,
    \qquad
    \mW\in\R^{2048\times196},
    \qquad
    \vb\in\R^{196}.
\]
Equivalently, for row-wise batches,
\[
    \mZ=\mH\mW+\mathbf{1}\vb^\top.
\]
The number of trainable parameters is
\[
    2048\cdot196+196=401604.
\]
We minimize the mean multiclass softmax
cross-entropy using mini-batches.

\paragraph{Methods.}
Adam~\citep{kingma2014adam} provides a first-order baseline.
SGN~\citep{gargiani2020promise} uses the full softmax GGN
with truncated parameter-space CG.
FGN uses the retained true-vs-rest curvature from
Theorem~\ref{thm:fgn_ggn_decomposition}.
For this affine head, the FGN row-space Gram has a closed form.
Let $(\va_i,\va_j)$ denote the logit-space margin gradients for examples $(i,j)$,
expressed in the same class ordering.
The unwhitened Gram is
\[
    G_{ij}
    =
    (\vh_i^\top\vh_j+1)(\va_i^\top\va_j),
\]
where $+1$ accounts for the bias.
The whitened row Gram is
\[
    K_{ij}
    =
    \sqrt{q_iq_j}\,G_{ij},
    \qquad
    q_i=p_\star^{(i)}p_\dagger^{(i)}.
\]
This closed-form expression is specific to the affine-head setting
and introduces no additional curvature approximation.

\paragraph{Hyperparameters.}
Hyperparameters are selected on a 
train-only validation split and fixed for final test evaluation.
The selected configurations are
\[
\begin{array}{lll}
\text{Adam:} & \alpha=3\cdot10^{-4}, & \\
\text{SGN:}  & \alpha=0.1,\ \lambda_0=1.0, &
\text{constant damping, CG maxiter }5,\\
\text{FGN:}  & \alpha=0.1,\ \lambda_0=1.0, &
\text{constant damping, CG maxiter }5.
\end{array}
\]
Both second-order methods use CG tolerance $10^{-5}$.
SGN uses warm-started CG, 
FGN uses row-space CG without warm-starting.

\paragraph{Final runs and timing.}
We train each method for $150$ epochs ($9600$ steps) across $10$ seeds, 
with batch size $128$ and evaluation every $100$ steps.
Each seed determines head initialization and a single shuffle
of the training examples; 
batches remain fixed across epochs.
The final batch wraps $48$ examples from the beginning
of the shuffled training array, giving $64$ batches per epoch.

Test accuracy and cross-entropy are interpolated onto a
$240$-point wall-time grid ending at $7.27$ s,
the shortest final time among the $30$ runs.
Curves show means and $\pm1$ standard deviation over seeds.
Threshold times are computed from the interpolated mean accuracy curve.
Wall time includes synchronized head updates but 
excludes feature extraction, compilation, evaluation, plotting, and CSV writes.

\begin{table}[t]
\centering
\small
\setlength{\tabcolsep}{4.5pt}
\caption{
\textbf{Frozen-feature head optimization on Cars196.}
Test accuracy and cross-entropy at the shared $7.27$ s horizon,
reported as mean $\pm1$ standard deviation over $10$ seeds. 
Threshold times are in seconds.
}
\label{tab:headopt_cars196}
\begin{tabular}{lccccc}
\toprule
Method
& Acc. (\%)
& CE
& Time to 42\%
& Time to 43\%
& Time to 44\% \\
\midrule
Adam
& $42.28 \pm 0.13$
& $2.8952 \pm 0.0076$
& $1.146$
& --
& -- \\
SGN
& $44.33 \pm 0.10$
& $2.3634 \pm 0.0006$
& $1.133$
& $1.968$
& $4.441$ \\
FGN
& $44.43 \pm 0.07$
& $2.3579 \pm 0.0005$
& $0.833$
& $1.414$
& $3.449$ \\
\bottomrule
\end{tabular}
\end{table}

\paragraph{Optimizer steps.}
FGN reaches $44\%$ mean test accuracy after approximately
$2443$ steps, compared with $2533$ for SGN.
Both finish near $44.5\%$ after $9600$ steps.
The wall-time advantage therefore comes primarily from cheaper updates, 
not substantially faster progress per step.

\paragraph{CG budget sensitivity.}
We repeat the $9600$-step runs over three seeds
with CG caps $\{1,2,5,10,20\}$,
keeping all other settings fixed.

\begin{table}[t]
\centering
\small
\setlength{\tabcolsep}{4pt}
\caption{
\textbf{Cars196 CG budget sensitivity.}
Final test accuracy, mean CG iterations, and ratio of
mean SGN/FGN optimizer step times over three seeds.
Paired entries are FGN / SGN.
}
\label{tab:cars196_cg}
\begin{tabular}{rccc}
\toprule
CG cap
& Test accuracy (\%)
& CG iterations
& SGN/FGN time \\
\midrule
1  & $44.26 / 43.99$ & $1.00 / 1.00$   & $1.14$ \\
2  & $44.33 / 44.30$ & $2.00 / 2.00$   & $1.17$ \\
5  & $44.51 / 44.55$ & $5.00 / 5.00$   & $1.35$ \\
10 & $44.52 / 44.50$ & $8.83 / 9.15$   & $1.47$ \\
20 & $44.51 / 44.48$ & $11.40 / 12.74$ & $1.62$ \\
\bottomrule
\end{tabular}
\end{table}

At a CG cap of five, both methods exhaust the iteration budget
in all solves.
SGN performs one additional operator application per solve
to accommodate warm starting, so equal CG caps do not imply equal
computational work.
Beyond five iterations, 
increasing the CG budget brings no consistent improvement in final test accuracy.

\paragraph{Competitor dispersion.}
We record competitor dispersion $\xi$,
the effective competitor count $N_{\mathrm{eff}}=1/(1-\xi)$,
and the dropped logit-space trace $p_\dagger\xi$
over all $8144$ training examples,
at initialization, every $500$ steps,
and at the final iterate, across $10$ seeds.
The effective competitor count is computed per example
before averaging.

At initialization, mean $\xi\approx0.995$
and $N_{\mathrm{eff}}=194.98$.
At the final iterate, mean $N_{\mathrm{eff}}$ is $10.33$
for SGN and $10.94$ for FGN.
The largest competitor accounts for approximately $30\%$
and $28\%$ of the conditional competitor mass, respectively.
For FGN, the dropped term still accounts for $40.85\%$
of the full parameter-space GGN trace.

Therefore, in this setting, 
matching SGN's test accuracy does not require 
a nearly one-hot competitor distribution 
or a negligible discarded curvature term.

\subsection{End-to-end training}
\label{apx:e2e}

\paragraph{Data and model.}
We train a CIFAR-style ResNet-20 with $278{,}324$ parameters
on CIFAR-100 ($50{,}000$ training and $10{,}000$ test images).
Training augmentation includes padding, random cropping, and horizontal flips.

\paragraph{Training protocol.}
All methods train for $200$ epochs with batch size $128$,
corresponding to $78{,}125$ updates when epochs are counted
by processed examples.
The initial learning rate is $0.1$ for SGD-M and $0.3$
for SGN and FGN, with tenfold reductions at epochs $100$ and $150$.
All methods use Nesterov momentum $0.9$ and kernel-only
weight decay $5\times10^{-4}$.
Decay is applied before momentum for SGD-M and after
momentum for SGN and FGN.

\paragraph{Curvature solvers.}
SGN applies the full softmax GGN with warm-started
parameter-space CG.
FGN uses the retained true-vs-rest curvature
with the lifted row-space CG solver.
Both methods use damping $\lambda=1$, at most $5$ CG iterations,
and tolerance $10^{-5}$.

\paragraph{Evaluation.}
We use five seeds and evaluate training and test metrics
at initialization and every $5$ epochs.
Test metrics cover all $10{,}000$ test images.
Training cross-entropy is measured on a fixed, unaugmented
subset of $5000$ training images,
with BatchNorm in inference mode.
We report the final epoch rather than the best test checkpoint.
Results are aggregated as mean $\pm$ sample standard deviation
across seeds.

\paragraph{Timing.}
Step times include synchronized training and input overhead,
but exclude compilation, evaluation, logging, and the first
input wait of each timing segment.
Within each run, we compute the mean step time per segment
and take the median across segments.
The resulting values are aggregated across seeds.
Experiments run on an NVIDIA RTX A4000.

\begin{table}[t]
\centering
\small
\setlength{\tabcolsep}{4pt}
\caption{
\textbf{ResNet-20 training on CIFAR-100.}
Final metrics after $200$ epochs and median training step times.
Entries are means $\pm1$ standard deviation over $5$ seeds.
Training CE is evaluated on the fixed training subset.
}
\label{tab:cifar100_main}
\begin{tabular}{lcccc}
\toprule
Method & Test accuracy (\%) & Test CE & Training CE & Step time (ms) \\
\midrule
SGD-M & $68.23\pm0.34$ & $1.288\pm0.014$ & $0.2584\pm0.0054$ & $10.639\pm0.540$ \\
SGN & $66.78\pm0.23$ & $1.666\pm0.020$ & $0.2142\pm0.0050$ & $95.503\pm1.164$ \\
FGN & $67.05\pm0.35$ & $1.459\pm0.032$ & $0.2372\pm0.0042$ & $76.154\pm1.124$ \\
\bottomrule
\end{tabular}
\end{table}

%%%%%%%%%%%%%%%%%%%%%%%%%%%%%%%%%%%%%%%%%%%%%%%%%%%%%%%%%%%%%%%%%%%%%%%%%%%%%%%
%%%%%%%%%%%%%%%%%%%%%%%%%%%%%%%%%%%%%%%%%%%%%%%%%%%%%%%%%%%%%%%%%%%%%%%%%%%%%%%

\end{document}